\documentclass[twoside,11pt]{article}

\usepackage{blindtext}

\usepackage[preprint]{jmlr2e}

\usepackage{amsmath}
\usepackage{booktabs}
\usepackage{mathrsfs}
\usepackage{float}
\usepackage{enumitem}
\usepackage{setspace}
\usepackage{subcaption}
\setcitestyle{numbers,square}
\newcommand{\indep}{\perp \!\!\! \perp}
\newtheorem{assumption}{Assumption}

\newtheorem{assumptionA}{Assumption}

\usepackage{amsmath,amsfonts,bm}
\usepackage{bbm}

\DeclareMathAlphabet{\mathsfit}{\encodingdefault}{\sfdefault}{m}{sl}
\SetMathAlphabet{\mathsfit}{bold}{\encodingdefault}{\sfdefault}{bx}{n}

\def\sE{{\mathbb{E}}}

\def\sR{{\mathbb{R}}}

\def\0{{\bf 0}}
\def\1{{\bf 1}}

\newcommand{\Var}{\mathrm{Var}}

\usepackage{lastpage}
\jmlrheading{}{2023}{1-\pageref{LastPage}}{}{}{}{}

\firstpageno{1}

\begin{document}

\title{Constructing Synthetic Treatment Groups without the Mean Exchangeability Assumption}

\author{\name Yuhang Zhang  \\
       \addr  {\email  zhangyh19@pku.edu.cn} \\
       School of Mathematical Sciences\\
       Peking University\\
       Beijing, China
       \AND
       \name Yue Liu  \\
       \addr {\email liuyue\_stats@ruc.edu.cn} \\
       Center for applied statistics and school of statistics\\
       Renmin University of China\\
       Beijing, China
       \AND
       \name Zhihua Zhang  \\
       \addr {\email zhzhang@math.pku.edu.cn} \\
       School of Mathematical Sciences\\
       Peking University\\
       Beijing, China}


\maketitle

\begin{abstract}

The purpose of this work is to transport the information from multiple randomized controlled trials to the target population where we only have the control group data. Previous works rely critically on the mean exchangeability assumption. However, as pointed out by many current studies, the mean exchangeability assumption might be violated. Motivated by the synthetic control method, we construct a synthetic treatment group for the target population by a weighted mixture of treatment groups of source populations. We estimate the weights by minimizing the conditional maximum mean discrepancy between the weighted control groups of source populations and the target population. We establish the asymptotic normality of the synthetic treatment group estimator based on the sieve semiparametric theory. Our method can serve as a novel complementary approach when the mean exchangeability assumption is violated. Experiments are conducted on synthetic and real-world datasets to demonstrate the effectiveness of our methods.
\end{abstract}


\section{Introduction}
The crucial goal of causal inference is to generalize the scientific findings beyond the source populations to the target population. The question of transportability touches on the heart of scientific methodology \citep{pearl2019new}. For example, we want to assess the effectiveness of a new treatment such as a new drug or a new policy, compared with a traditional treatment in a target population of interest. Suppose that we have conducted several randomized controlled trials (RCTs) in different situations. However, it is infeasible to carry out an RCT in the target population due to some practical issues such as the budget constraint. Instead, we only have some historical data about the control group from the target population.
Although we can estimate the average treatment effect in each trial, the target population has some different properties from the population underlying any one trial due to the inclusion and exclusion criteria. Without the treatment group data from the target population, the problem of how to transport the information from the source populations to the target population becomes a very challenging issue.

Most of the previous works (e.g., \citep{Rudolph2017RobustEO,Dahabreh2019GeneralizingCI,Dahabreh2020ExtendingIF,li2021improving}) depend critically on the mean exchangeability assumption that the target population and the source populations share the same potential outcome expectation. Under this assumption, one can simply pool the data from the source populations and estimate the conditional average treatment effect from the pooled dataset \citep{dahabreh2020towards}. However, many studies (e.g., \citep{gechtery2015generalizing,huitfeldt2018choice,dahabreh2019sensitivity,dahabreh2022global}) show that the mean exchangeability assumption is controversial and needs to be supported by the substantive background knowledge. When this assumption is violated by the empirical evidence, the heterogeneity between the source populations and the target population leads to the site selection bias \citep{allcott2015site} which prevents us from applying the classical transportability method. 
To overcome this challenge, we need to find an alternative connection between the source populations and the target population. Otherwise, if the target population differs from the source populations in an arbitrary way, we could not utilize the information from the source populations even if we have an infinite amount of data from the source populations.

When the outcome variable is binary, \citet{huitfeldt2018choice} and \citet{cinelli2021generalizing} assumed the homogeneity of COST parameters (the invariance of probabilities of causation) to transport the information. In this paper we generalize their methods from binary outcomes to continuous outcomes and propose alternative identification assumptions to the mean exchangeability assumption. The motivation comes from the deep connection between the transportability problem and the synthetic control method \citep{Abadie2010SC}. These new identification assumptions serve as a bridge between the source populations and the target population, enabling information to be transported. Based on them, we create a synthetic treatment group for the target population by weighting the treatment groups of the source populations. By combining the synthetic treatment group and the control group data from the target population, we can estimate the average treatment effect under the violation of the mean exchangeability assumption. 

The contributions in this paper can be summarized as follows:
\begin{itemize}[nosep]
	\item We explore the deep connection between the transportability problem and the synthetic control method,  leading to two alternative identification assumptions. Based on them, we can generalize the identification formula that depends on the mean exchangeability assumption.
	\item We propose a complementary approach to achieve transportability under the violation of the mean exchangeability assumption by constructing a synthetic treatment group.
	\item We estimate the weights by minimizing the conditional maximum mean discrepancy (CMMD) between the weighted control group distribution of the source populations and the target population, which circumvents the difficult problem of estimating the conditional probability density function. 
	\item Formulating the problem as a sieve semiparametric two-step GMM estimation \citep{chen2015sieve}, we derive the asymptotic distribution of the treatment group expectation estimator and construct the asymptotic variance estimator.
\end{itemize}
The remainder of this paper is organized as follows. In section 2 we present a review of the previous literature. In Section 3 we discuss the identification assumptions in the previous work. In Section 4 we propose two alternative identification assumptions and the identification formula with a summary in Table \ref{table:summary of the identification assumptions}. In Section 5 we introduce the weight estimation method by minimizing the CMMD. In Section 6 we apply the sieve semiparametric two-step estimation to establish the asymptotic normality of the synthetic treatment group estimator. We present the experimental results in Section 7 and finally conclude our work.

\section{Related Work}

As far as we know, when the outcome variable is continuous, there is no prior work that identified the average treatment effect of the target population under the violation of the mean exchangeability assumption. \citet{gechtery2015generalizing} assumed the invariance of copula functions to construct a bound for the average treatment effect of the target population.  

When the outcome variable is binary, in order to transport the information under the violation of the mean exchangeability assumption, \citet{huitfeldt2018choice} proposed the counterfactual outcome transition state (COST) parameters which were also called by \citet{tian2000probabilities} and \citet{cinelli2021generalizing} as the probabilities of causation. Henceforth, we will use these two terms interchangeably.  \citet{shingaki2023probabilities} extended the result of \citet{cinelli2021generalizing} from the experimental study to the observational study. In addition, \citet{colnet2023risk} also discussed the COST parameters in their Lemma 7, Definitions 6 and 7. According to \citet{huitfeldt2019effect}, the important limitation of COST parameters is that they are only defined with respect to the binary outcome and the generalization to the continuous outcome is nontrivial. To the best of our knowledge, our work is the first attempt to generalize the COST parameters from the binary outcome variable to the continuous outcome variable and incorporate the covariate information. 

When the violation of the mean exchangeability assumption is caused by the unobserved effect modifiers, \citet{nguyen2017sensitivity,nguyen2018sensitivity}, \citet{nie2021covariate} and \citet{huang2022sensitivity} conducted the sensitivity analysis to evaluate the possible impact of the violation. Without requiring detailed background knowledge about the relationship between the unobserved effect modifiers and the observed variables, \citet{dahabreh2019sensitivity,dahabreh2022global} used the bias function and the exponential tilt model to conduct the sensitivity analysis. 

The difference-in-differences method (DID) \citep{Card1994MinimumWA} and the synthetic control method (SCM) \citep{Abadie2010SC} are two widely used methods of treatment effect estimations in comparative case studies. Their identification assumptions are closely connected with the ones underlying the transportability method. In Section \ref{section:conditional parallel trends assumption}, we will discuss the connection between the classical transportability method and the difference-in-differences method. Furthermore, we will elaborate on the relationship between our proposed method and the synthetic control method. 

Most of the studies on the synthetic control method deal with the panel data. Only a few works (e.g., \citep{wong2015three,gunsilius2020distributional,Chen2020ADS}) analyze the repeated cross-sectional data. \citet{wong2015three} estimated the weights by moment matching using the first moment information with multiple pretreatment periods. The distributional synthetic control method constructed the linear mixture of the donor pool’s quantile functions as the Wasserstein barycenter to estimate the treatment unit’s counterfactual quantile function \citep{gunsilius2020distributional}. \citet{shi2022assumptions} investigated the assumptions of the synthetic control from the perspective of the individual level data. 

Other than the synthetic control method, there are many works that estimate the target quantity by creating a weighted average of counterparts from the source populations. However, in most situations, their weights are of population level and do not make use of the covariate information. \citet{hasegawa2017myth} specified the weight for each source population to match the covariate distribution of the target population and constructed a weighted average treatment effect estimator. \citet{dalalyan2018optimal} constructed a convex combination of multiple probability density functions to estimate the target distribution by minimizing the Kullback-Leibler divergence between the mixture distribution and the empirical distribution.

The weight estimation also plays a key role in the problem of mixture proportion estimation (e.g., \citep{hall1981non,yu2018efficient}), class ratio estimation (e.g., \cite{iyer2014maximum,iyer2016privacy}) and class-prior estimation (e.g., \cite{kawakubo2016computationally,du2014semi}). With access to samples from the target distribution and each component distribution, they estimate the finite dimensional weight vector by minimizing a statistical distance between the weighted component distribution and the target distribution. In contrast, we face a more difficult challenge because we don't have access to the sample directly from each mixture component distribution and we utilize the covariate information. As a result, we need to estimate several infinite dimensional weight functions.

Our work also relates to the problem of parameter estimation by minimizing the maximum mean discrepancy.  \citet{iyer2014maximum}, \citet{kawakubo2016computationally} and \citet{yu2018efficient} estimated the mixture proportion weights by minimizing the maximum mean discrepancy between the target empirical distribution and the weighted empirical component distributions. \citet{cherief2022finite} estimated the target probability distribution function by a linear mixture of probability distribution functions from a known dictionary. \citet{Ren2016ConditionalGM,ren2021improving} adopted this method to train deep generative models. \citet{alquier2020universal} investigated the problem of universal robust regression while leaving the covariate distribution unspecified. 

We employ the sieve semiparametric two-step estimation method to establish the asymptotic normality of the synthetic treatment group estimator \citep{chen2015sieve}. The sieve extremum estimation is widely used for the semiparametric models where the parameter of interest is finite dimensional and the nuisance parameter is infinite dimensional \citep{chen2007large}. \citet{chen2003estimation} and \citet{ichimura2010characterization} also investigated the problem of semiparametric two-step estimation but they only derived the closed-form solution of the asymptotic variance estimator for some special examples. If the nuisance function can be estimated by the conditional moment restriction models in the first step, then one can formulate the problem as the conditional moment restriction model with
different conditioning variables or the sequential moment restrictions model. Consequently, one can use the sieve minimum distance estimator to estimate the parameter of interest and establish the asymptotic normality of the resulting estimator \citep{ai2007estimation,ai2012semiparametric}.

Many studies have made great contributions to the field of transportability using the causal graphical model \citep{pearl2022external}. However, in this paper we mainly approach the transportability problem from the perspective of the potential outcome framework. Please refer to \citet{degtiar2023review} and \citet{lesko2020target} for literature reviews about the problem of transportability.

\section{Preliminaries}
In this section we will first introduce the problem setup and the notations. Then we will review the main identification assumptions underlying the previous works.
\subsection{Problem Setup}
Suppose that we conduct $N$ RCTs to estimate the effect of a new treatment compared with the traditional treatment. For the target population of interest, we only have access to the historical control group data. Let $ D $ denote the population indicator with $ D=0 $ for the target population and $D=1, \ldots, N$ for the source populations. We refer to the population underlying the $i$-th RCT as the $i$-th source population. We denote by $A$  the treatment assignment indicator with $A=1$ for the treatment group and $A=0$ for the control group. Let $Y\in \mathbb{R}$ and $X \in \mathbb{R}^d$ denote the outcome and the covariate. We refer to $ Y(a) $ as the potential outcome under the treatment assignment $ a\in \{0,1\} $. For the $i$-th trial, the data consists of independent and identically distributed random vectors $(D_j=i, A_j, X_j, Y_j)$, $j=1, \ldots, m_i+n_i$, where $m_i$ is the number of units in the treatment group and $n_i$ is the number of units in the control group. For the target population, the sample consists of independent and identically distributed random vectors $(D_j=0, A_j=0, X_j, Y_j)$, $j=1, \ldots, n_0$, where $n_0$ is the sample size of the target population. Note that in the target population, every unit is in the control group. 
The average treatment effect of the target population can be defined as:
\begin{align*} 
	{ATE} \triangleq {\mathbb E}[Y(1)-Y(0)|D=0] = {\mathbb E}[Y(1)|D=0] - {\mathbb E} [Y(0)|D=0].
\end{align*}
Since all the units in the target population belong to the control group, we have ${\mathbb E}[Y(0)|D=0]={\mathbb E}[Y|D=0]$. However, due to the absence of the treatment group data in the target population, the challenge is to address the identifiability assumptions
about the treatment group expectation $\theta_0={\mathbb E}[Y(1)|D=0]$.


Throughout this paper, we maintain the stable unit treatment value assumption (SUTVA) that there is no interference between study units and different versions of the treatment are irrelevant. For any unit in the $i$-th trial, if $A=a$, then $Y=Y(a)$. Every unit in the target population is in the control group. Furthermore, one always makes the strong ignorability assumption \citep{rosenbaum1983central}, that is, 
\[
Y(a)\indep A|[X, D=i] \mbox{ and }  \Pr(A=a|X=x,D=i)>0, \mbox{ where } a\in \{0,1\}   \mbox{ and } i\in\{1, \ldots, N\}.
\]
\subsection{Main Assumptions of Previous Works}\label{section: Key assumptions of previous works}
Besides the SUTVA assumption and the strong ignorability assumption, most of the previous works depend critically on the mean exchangeability assumption that the conditional expectation of potential outcome remains the same between the $i$-th trial and the target population:
\begin{align*} 
	{\mathbb E}[Y(a)|X=x,D=i]={\mathbb E}[Y(a)|X=x,D=0],\mbox{ for } a\in \{0,1\}   \mbox{ and } i\in\{1, \ldots, N\}.
\end{align*}
Many studies (e.g., \cite{hotz2005predicting,lesko2017generalizing,lu2019causal,Li2020TargetPS,Egami2020ElementsOE}) employ a stronger distribution exchangeability assumption:
\begin{align*} 
	Y(a)\indep D|X.	
\end{align*}
Some previous works (e.g., \citep{allcott2015site,dahabreh2020towards}) relax the mean exchangeability assumption and make a weaker assumption: the conditional exchangeability in measure assumption. They assume that the source populations and the target population share the same conditional average treatment effect:
\begin{align}\label{equation: conditional exchangeability in measure}
	{\mathbb E}[Y(1)-Y(0)|X=x,D=i]={\mathbb E}[Y(1)-Y(0)|X=x, D=0], \mbox{ for } i\in\{1, \ldots, N\}. 	
\end{align}
In some situations, the conditional exchangeability in measure assumption might still be violated. For example, when different hospitals have different medical conditions or there are some unmeasured effect modifiers whose distributions vary among different populations, the conditional average treatment effects may differ from site to site \citep{nie2021covariate,huang2022sensitivity}.

When the outcome variable $Y$ is binary, the counterfactual probabilities: 
\begin{align*}
	\Pr(Y(1)=a | Y(0)=b) \mbox{ with } a,b\in\{0,1\},
\end{align*}
are named by \citet{huitfeldt2018choice} as the COST parameters. They are also named by \citet{tian2000probabilities} as the probabilities of causation. According to \citet{huitfeldt2019effect}, the COST parameters overcome many shortcomings of other effect measures (e.g., the risk difference based on the mean exchangeability assumption). According to Table 1 of \citet{huitfeldt2018choice}, the COST parameters have no zero constraints or baseline risk dependence, are collapsible over baseline covariates and do not lead to logically invalid results. Besides, they also have the scale invariant property. In contrast, the scale dependence of the mean exchangeability assumption restricts its use in settings with non-Gaussian and discrete outcomes according to \citet{ding2019bracketing}. 

\citet{huitfeldt2018choice} reasoned that the homogeneity of the COST parameters between the source populations and the target population is more plausible than the mean exchangeability assumption under many data generating mechanisms. Analogously, \citet{cinelli2021generalizing} assumed the invariance of probabilities of causation across populations in place of the mean exchangeability assumption. 

The assumptions of the homogeneity of the COST parameters or the invariance of probabilities of causation can be formulated as
\begin{align*}
	Y(1) \indep D|Y(0).
\end{align*}
According to \citet{huitfeldt2019effect}, the COST parameters have only been limited to the binary outcome variable and it is difficult to generalize them to the continuous outcome variables. The reason is that they need to additionally rely on the monotonicity assumption in order to identify the COST parameters. The monotonicity assumption can be expressed as
\begin{align*}
	Y(1)\geq Y(0).
\end{align*}

\citet{cinelli2021generalizing} constructed bounds for the probabilities of causation when the monotonicity assumption is violated. By combining the information from multiple source populations, they identified the probabilities of causation without the monotonicity assumption. They estimated the COST parameters by solving a system of linear equations implied by the law of total probability. Nevertheless, it is difficult to generalize their method to the continuous outcome variables because it would require an infinite number of source populations which is infeasible in practice. To make things worse, for the continuous outcome variable, even the monotonicity assumption is too weak to point-identify the COST parameters. 


\subsection{Connections Between the Conditional Exchangeability in Measure Assumption and the Conditional Parallel Trends Assumption}\label{section:conditional parallel trends assumption}

The conditional exchangeability in measure assumption has a strong connection with the conditional parallel trends assumption \citep{nie2019nonparametric,callaway2021difference} underlying the difference-in-differences method (see Footnote 8 of \citet{gechtery2015generalizing}). To illustrate this, we will first introduce the difference-in-differences method following the example in \citet{nie2019nonparametric}. Suppose that there are two comparable cities. The first city is subject to a policy intervention while the second city is not. We observe $n$ i.i.d random vectors $\left(D_i, T_i, X_i, Y_i\right)$. The city indicator $D_i \in\{0,1\}$ denotes whether the $i$-th individual is in the exposed or control city ($D_i=0$ for the exposed city and $D_i=1$ for the control city). $T_i \in\{0,1\}$ denotes the time indicator ($T_i=0$ for the pre-treatment period and $T_i=1$ for the post-treatment period). $X_i$ is the covariate variable and $Y_i$ is the outcome of interest. Let $Y_{T=t}(0)$ and $Y_{T=t}(1)$ denote the control and treated potential outcomes at time $t$.

Because the counterfactual expectation $\mathbb{E}\left[Y_{T=1}(0) \mid X=x, D=0\right]$ is unobserved, the DID method imposes the conditional parallel trends assumption that the counterfactual control potential outcome of the treated city will follow the parallel trend with the control city:
	\begin{align*}
		&\mathbb{E}\left[Y_{T=1}(0)| X=x, D=1\right]-\mathbb{E}\left[Y_{T=0}(0)| X=x, D=1\right] \\
		& = \mathbb{E}\left[Y_{T=1}(0)| X=x, D=0\right]-\mathbb{E}\left[Y_{T=0}(0)| X=x, D=0\right].
	\end{align*}
If we substitute the role of treatment indicator $A$ for the time indicator $T$ and suppress the potential outcome notation "$(0)$", then we will find that the conditional exchangeability in measure assumption is mathematically equivalent to the conditional parallel trends assumption in Eqn.~\ref{equation: conditional exchangeability in measure} such that 
\begin{align*}
	&\mathbb{E}\left[Y_{A=1}| X=x, D=1\right]-\mathbb{E}\left[Y_{A=0}| X=x, D=1\right] \\
	& = \mathbb{E}\left[Y_{A=1}| X=x, D=0\right]-\mathbb{E}\left[Y_{A=0}| X=x, D=0\right].
\end{align*}
As a result, any limitations of the conditional parallel trends assumption could also be applied to the conditional exchangeability in measure assumption. When the conditional parallel trends assumption is violated, the difference-in-differences method is no longer valid. Following the same logic, when the conditional exchangeability assumption is violated, the classical transportability method is no longer valid.

\section{Identification Without the Mean Exchangeability Assumption}\label{section: Identification without the mean exchangeability assumption}

According to \citet{gunsilius2020distributional}, the synthetic control method can serve as a natural complementary approach to the difference-in-differences method when the conditional parallel trends assumption is violated. This inspires us to investigate the underlying assumptions of the synthetic control method and propose complementary identification assumptions for the mean exchangeability assumption. 
\subsection{The Invariance of Counterfactual Expectation Transition Mechanism}
The synthetic control method depends on the mean independence assumption \citep{kellogg2021combining,ONeill2016EstimatingCE,angrist2009mostly,ding2019bracketing} that individuals with similar covariates and outcomes in the pre-treatment periods will tend to have similar counterfactual control potential outcomes in the post-treatment period:
	\begin{align*}
		E\left[Y_{T=1}(0)| Y_{T=0}(0)=y,X=x, D=i\right]
		=E\left[Y_{T=1}(0)| Y_{T=0}(0)=y,X=x, D=0\right],
	\end{align*}
	where $i\in\{1, 2, \ldots, N\}.$
	
Following the same logic in Section \ref{section:conditional parallel trends assumption}, we propose the following assumption by substituting the time $T$ by the treatment indicator $A$ and suppressing the common potential outcome notation "$(0)$". 

\begin{assumption}[The Invariance of CETM]
	\label{assumption: the invariance of counterfactual expectation transition mechanism}
	For every covariate $x\in \mathbb{R}^d$ and outcome $y\in \mathbb{R}$, the counterfactual expectation transition mechanism (CETM) remains the same between the source populations and the target population. That is, 
	\begin{align*} 
		{\mathbb E}[Y(1)|Y(0)=y,X=x,D=i]={\mathbb E}[Y(1)|Y(0)=y,X=x,D=0], \mbox{ for }	i\in\{1, \ldots, N\}.
	\end{align*}
\end{assumption}
Rather than concerning about the marginal distributions of $Y(0)$ and $Y(1)$ separately, we assume that the relationship between $Y(0)$ and $Y(1)$ stays invariant between the study populations and the target population. Different from the mean exchangeability assumption, Assumption \ref{assumption: the invariance of counterfactual expectation transition mechanism} not only conditions on the covariates $X$ but also includes the potential outcome $Y(0)$. The counterfactual expectation transition mechanism $\mathbb{E}[Y(1)|Y(0)=y,X=x]$ generalizes the COST parameters from the binary outcome variable to the continuous outcome variable and incorporates the covariate information. 

\citet{huitfeldt2019effect} stated that few examples of data generating processes leading to the homogeneity of an effect measure exist in the literature and finding plausible mechanisms may be infeasible in many situations. Fortunately, according to them, the COST parameter based on Assumption \ref{assumption: the invariance of counterfactual expectation transition mechanism} is a notable exception. According to \citet{cinelli2021generalizing}, if the outcome variable is the product of several independent factors and only some of them are different across populations, then the invariance of probabilities of causation may be entailed. It is consistent with the fact that Assumption \ref{assumption: the invariance of counterfactual expectation transition mechanism} can be implied by the independent causal mechanism assumption of \citet{shi2022assumptions} if the covariate variables $X$ and the control potential outcomes $Y(0)$ can serve as the causes. The independent causal mechanism assumption entails that the expectation of the treatment potential outcome conditional on the control potential outcome and the baseline covariate is independent of which population the individual is from. 

Assumption \ref{assumption: the invariance of counterfactual expectation transition mechanism} is more plausible than the mean exchangeability assumption in some situations because the potential outcome $Y(0)$ contains additional information other than that of the covariate $X$. For example, patients in different hospitals with similar baseline covariates $X$ and control group potential outcomes $Y(0)$ tend to have access to similar level of medical care service and baseline health conditions, which will lead to similar treatment group potential outcomes $Y(1)$. Assumption \ref{assumption: the invariance of counterfactual expectation transition mechanism} naturally holds if the potential outcome $Y(1)$ is generated by the potential outcome $Y(0)$ and the covariate $X$ through a common mechanism $m(\cdot)$ shared by the source populations and the target population. That is, 
\begin{align*} 
	Y(1)=m(Y(0),X)+\epsilon, \mbox{ where } {\mathbb E}[\epsilon|Y(0),X,D]=0.		
\end{align*}

When the outcome variable is binary, there are several real data applications based on Assumption \ref{assumption: the invariance of counterfactual expectation transition mechanism} from previous works. \citet{cinelli2021generalizing} studied three experiments to determine the effects of Vitamin A supplementation on childhood mortality across different regions. For the problem of non-compliance, \citet{takahata2018identification} adopted Assumption 1 with $D$ as the compliance indicator to estimate $\mathbb{E}[Y(1)-Y(0)|Y(0)]$. They investigated the National Job Training Partnership Act study and found that the empirical findings were consistent with Assumption \ref{assumption: the invariance of counterfactual expectation transition mechanism}. In addition, there are numerous examples that illustrate Assumption \ref{assumption: the invariance of counterfactual expectation transition mechanism} from the literature. \citet{huitfeldt2018choice,huitfeldt2019effect,huitfeldt2021shall} studied the effect of antibiotic treatment on mortality in patients with a specific bacterial infection, the adverse effects of Codeine and the effect of treatment with Penicillin on the risk of anaphylaxis. They believed that applications of Assumption \ref{assumption: the invariance of counterfactual expectation transition mechanism} could occur with some frequency when studying the effectiveness and safety of pharmaceuticals. In addition, \citet{cinelli2021generalizing} presented an illustrative example about the effect of playing Russian Roulette on mortality. 

Figure \ref{fig: causal graph} depicts the causal graph of a data generating process which satisfies Assumption \ref{assumption: the invariance of counterfactual expectation transition mechanism}. There are other causal graphs which also imply Assumption \ref{assumption: the invariance of counterfactual expectation transition mechanism}. Please refer to Figure 3 of \citet{cinelli2021generalizing} and Figure 6 of \citet{huitfeldt2021shall} for further details.
\begin{figure}
	\centering
	\includegraphics[width=0.3\textwidth]{"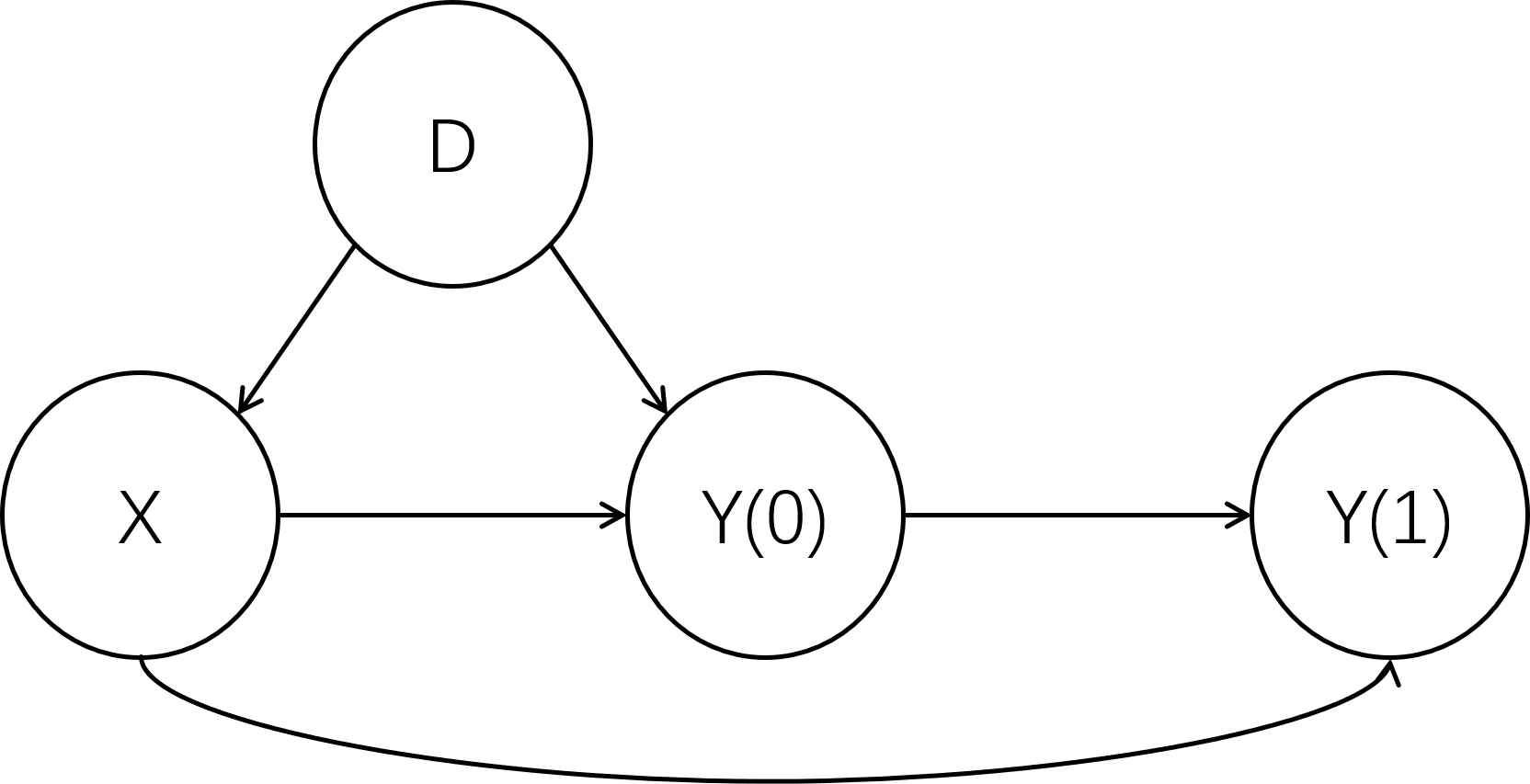"}
	\caption{A causal graph illustration for Assumption \ref{assumption: the invariance of counterfactual expectation transition mechanism}.}\label{fig: causal graph}
	\vspace{-10pt} 
\end{figure}
\subsection{The Conditional Mixture Distribution Assumption}
As mentioned in Section \ref{section: Key assumptions of previous works}, the monotonicity assumption for the continuous outcome variable is too weak to obtain the identification result. We propose an alternative identification assumption whose motivation comes from the synthetic control method that the distribution of the treated unit $P(Y_{T=0}|D=0)$ lies in
the convex hull of the distributions of control units $\{P(Y_{T=0}|D=i)\}_{i=1}^N$ \citep{wan2018panel,kato2023synthetic}:
	\begin{align*}
		P(Y_{T=0}\leq y|D=0)=\sum_{i=1}^{N}w_i P(Y_{T=0}\leq y|D=i)
	\end{align*}
Inspired by this assumption, we propose the following conditional mixture distribution assumption and utilize the covariate information. Similarly, we substitute the role of the treatment indicator $A$ for the time indicator $T$.
\begin{assumption}[Conditional Mixture Distribution]
	\label{asumption:mixture distribution}
	The control group distribution of the target population can be represented as a weighted mixture distribution constructed from the source populations. For every covariate $x\in \mathbb{R}^d$ and outcome $y\in \mathbb{R}$, there exists a set of weights $ \{w_i^*(x)\}_{i=1}^N $ such that  
	\begin{equation}\label{equation:assumption:mixture distribution}
		\Pr(Y(0) \leq y|X=x,D=0)=\sum_{i=1}^{N}w_i^*(x) \Pr(Y(0) \leq y|X=x,D=i).
	\end{equation}
\end{assumption}
\begin{proposition}\label{proposition: weaker than the distribution exchangeability assumption}
Assumption \ref{asumption:mixture distribution} is weaker than the distribution exchangeability assumption $Y(0)\indep D|X$. It can be implied by $ Y(0)\indep D|X $ with any weights $\{w^*_i(x)\}_{i=1}^N$ as long as $\sum_{i=1}^{N}w^*_i(x)=1$. If the covariate $X$ and the control potential outcome $Y(0)$ can serve as the minimal invariant set, then the assumptions of the stable distributions, the sufficiently similar donors and the target donors overlap of \citet{shi2022assumptions} together imply Assumption \ref{asumption:mixture distribution}.
\end{proposition}

When the outcome variable $Y$ is binary, Assumption \ref{asumption:mixture distribution} generally holds except in some degenerate cases. As a result, \citet{cinelli2021generalizing} implicitly adopted Assumption \ref{asumption:mixture distribution} in their Theorem 2 when they have multiple source populations. Similar assumptions have also been used in the fields of mixture proportion estimation \citep{yu2018efficient,hall1981non}, class ratio estimation \citep{iyer2014maximum,iyer2016privacy}, class prior estimation \citep{kawakubo2016computationally,du2014semi}, quantification \citep{castano2022matching,chow2022semiparametric} and multi-source domain adaptation \citep{zhang2015multi} although for different purposes. Assumption \ref{asumption:mixture distribution} is reasonable when all the control group distributions are mixture distributions and they share a common set of mixture components, which is the fundamental assumption underlying the field of topic modeling \citep{Lee2013UsingMS}. 
\begin{example}
	\label{example: mixture distribution}
	In order to assess the effectiveness of a new antihypertensive drug on the target population, researchers have conducted RCTs in $N$ different hospitals. Denote the outcome $y$ and the covariate $x$ as the blood pressure and the baseline health condition such as the Body Mass Index. Different hospitals have different age group distributions. We refer to $m\in\{1,2,3\}$ as children, the working-age population and the elderly population, respectively. And $\{a(m|x,i)\}_{m=1}^3$ denotes the proportions of different age groups in the $i$-th hospital ($i=0$ refers to the target population) with the baseline health condition $x$; $\{F(y|x,m)\}_{m=1}^3$  represents the control group distributions of different age groups with the baseline health condition $x$. The control group distributions of the source populations and the target population share a common set of mixture components $\{F(y|x,m)\}_{m=1}^3$. The only difference lies in their mixture weights $\{a(m|x,i)\}_{m=1}^3$:
	\begin{align*} 
		\Pr(Y(0) \leq y|X=x,D=i)&=\sum_{m=1}^{3}a(m|x,i)F(y|x,m), \mbox{ for }	i\in\{0, 1, \ldots, N\}.
	\end{align*}
	As long as the information from the source populations is sufficient enough, we can presume that there exists a set of weights $\{w^*_i(x)\}_{i=1}^N$ such that
	\begin{equation}\label{equantion:example 2}
		\begin{pmatrix}
			a(1|x,1) & a(1|x,2) & \cdots & a(1|x,N) \\
			a(2|x,1) & a(2|x,2) & \cdots & a(2|x,N) \\
			a(3|x,1) & a(3|x,2)& \cdots & a(3|x,N)
		\end{pmatrix}\begin{pmatrix}
			w^*_1(x) \\
			\vdots \\
			w^*_N(x)
		\end{pmatrix}=\begin{pmatrix}
			a(1|x,0) \\
			a(2|x,0)\\
			a(3|x,0)
		\end{pmatrix},
	\end{equation}
	then Assumption \ref{asumption:mixture distribution} holds. When we use the classical transportability method, the heterogeneity of the source populations increases the risk of violating the mean exchangeability assumption. However, when we consider Assumption \ref{asumption:mixture distribution}, the heterogeneity of the source populations leads to more diverse combinations of mixture weights such that the coefficient matrix in Eqn.~\ref{equantion:example 2} has a larger column rank. Consequently, Assumption \ref{asumption:mixture distribution} is more likely to be satisfied. Under Assumption \ref{asumption:mixture distribution}, we don't need to use the strong distribution exchangeability assumption that all the populations share the same control group distribution. We can combine the information from multiple heterogeneous source populations which can complement each other.
\end{example}
\subsection{The Identification Result}
The synthetic control method constructs the synthetic control unit for the treated unit by weighting the control units. Motivated by that, we can create a synthetic treatment group for the target population by weighting the treatment groups of the source populations. The synthetic control method estimates the weights by matching the weighted control unit with the treated unit in the pre-intervention period. Its effectiveness depends on the assumption that the optimal weights remain the same between the pre-intervention period and the post-intervention period. We can obtain a similar result that the optimal weights of our method stay invariant between the control group and the treatment group. 
\begin{theorem}
	\label{proposition: identification}
	Under Assumptions \ref{assumption: the invariance of counterfactual expectation transition mechanism} and \ref{asumption:mixture distribution}, for every covariate pattern $x\in \mathbb{R}$, we have
	\begin{equation}\label{equation: 1}
		{\mathbb E}(Y(1)|X=x,D=0)=\sum_{i=1}^{N}w_i^*(x){\mathbb E}(Y(1)|X=x,D=i) =\sum_{i=1}^{N}w_i^*(x)g^*_i(x),
	\end{equation}
	where $\{w_i^*(x)\}_{i=1}^N$ refers to the set of weights for the control group in Assumption \ref{asumption:mixture distribution} and $g^*_i(x)$ refers to the outcome regression function ${\mathbb E}(Y|X=x,A=1,D=i)$. Furthermore, we can marginalize over the covariate $X$ to obtain the identification formula for the treatment group expectation, that is, 
	\begin{equation}\label{equation: 2}
		{\mathbb E}[Y(1)|D=0]={\mathbb E}\left\{{\mathbb E}[Y(1)|X,D=0]|D=0\right\}={\mathbb E} \Big[\sum_{i=1}^{N}w_i^*(X)g^*_i(X)|D=0 \Big].
	\end{equation}	
\end{theorem}
\begin{proposition}\label{proposition: generalize the identification formula under the mean exchangeability assumption}
Even when Assumptions \ref{assumption: the invariance of counterfactual expectation transition mechanism} and \ref{asumption:mixture distribution} are violated, the identification results in Eqns. \ref{equation: 1} and \ref{equation: 2} still hold no matter how we choose the weights as long as the mean exchangeability assumption holds and $\sum_{i=1}^{N}w^*_i(x)=1$. From this perspective, Eqns. \ref{equation: 1} and \ref{equation: 2} generalize the identification formula under the mean exchangeability assumption \citep{dahabreh2020towards}.     
\end{proposition}

The essential difference between our work and the works of \citet{huitfeldt2018choice} and \citet{cinelli2021generalizing} is that they both required estimating the COST parameters before estimating the target parameter $\mathbb{E}[Y(1)|D=0]$, otherwise the target parameter is unidentified. However, when the outcome variable $Y$ is continuous, the CETM $\mathbb{E}[Y(1)|Y(0),X]$ can not be identified without further assumptions because $Y(1)$ and $Y(0)$ can not be observed at the same time. Assumption \ref{asumption:mixture distribution} circumvents this challenge and transforms this seemingly impossible estimation problem for the target population into the problem of weight estimations. Assumptions \ref{assumption: the invariance of counterfactual expectation transition mechanism} and \ref{asumption:mixture distribution} enable us to estimate $\mathbb{E}[Y(1)|D=0]$ even when the CETM $\mathbb{E}[Y(1)|Y(0),X]$ is unidentified.

In order to further elaborate on Example \ref{example: mixture distribution}, we construct the following example in which the identification formulae of the synthetic treatment group (Eqns. \ref{equation: 1} and \ref{equation: 2}) hold while the mean exchangeability assumption is violated.
\begin{example}[continued with Example \ref{example: mixture distribution}]
	\label{example: violation of mean exchangeability}
	Due to the common biological mechanism underlying the drug action, patients from the source populations and the target population share the same counterfactual expectation transition mechanism: ${\mathbb E}[Y(1)|Y(0)=y,X=x]=c(x)y+d(x)$. Here we denote $Y(1)$ and $Y(0)$ as the potential outcomes under the new hypertensive drug and the traditional treatment, respectively. The control group expectation of the $i$-th source population follows the linear model ${\mathbb E}[Y(0)|X=x,D=i]=a_ix+b_i$, for $i\in\{1, \ldots, N\}$. Because different hospitals have different baseline medical conditions, they have different control group distributions such that $a_i\neq a_j$ when $i\neq j$. Because both the treatment potential outcome and the conditional treatment effect depend on the baseline medical condition, we presume that $c(x)\neq 0$ and $c(x)\neq 1$.  
	
	As we will show in Appendix \ref{subsection: proof of example 2}, the heterogeneity between the source populations leads to the violation of the mean exchangeability assumption and the conditional exchangeability in measure assumption. As a result, the classical transportability method is inapplicable. However, according to Theorem \ref{proposition: identification}, we can still identify the treatment group expectation of the target population by constructing a synthetic treatment group. Please refer to Section \ref{section: Simulated experiment} for simulated experimental results which further support Example \ref{example: violation of mean exchangeability}.

\end{example}
For a better comparison with previous works, we summarize the identification assumptions in Table \ref{table:summary of the identification assumptions}. 
\begin{table}[h]
	\centering
	\caption{A summary of the identification assumptions.}
	\large
	\resizebox{\linewidth}{!}{
		\begin{tabular}{ll}\toprule
			Identification Assumption   & Formulation {\footnotesize ($Y$: outcome; $A$: treatment; $X$: covariate; $D$: population index)}   \\ \midrule
			Mean Exchangeability~\citep{Rudolph2017RobustEO,Dahabreh2019GeneralizingCI,Dahabreh2020ExtendingIF,li2021improving}   & ${\mathbb E}[Y(a)|X=x,D=i]={\mathbb E}[Y(a)|X=x,D=0] \mbox{ for }a\in\{0,1\}$ \\
			Conditional Exchangeability in Measure~\citep{allcott2015site,dahabreh2020towards} & $	{\mathbb E}[Y(1)-Y(0)|X=x,D=i]={\mathbb E}[Y(1)-Y(0)|X=x,D=0]$\\
			Distribution Exchangeability~\cite{hotz2005predicting,lesko2017generalizing,lu2019causal,Li2020TargetPS,Egami2020ElementsOE}& $Y(a)\indep D|X \mbox{ for }a\in\{0,1\}$\\ \midrule
			Homogeneity of the COST parameters (Invariance of probabilities of causation)~\citep{huitfeldt2018choice,cinelli2021generalizing} & $Y(1) \indep D|Y(0)$, where $Y(1)$ and $Y(0)$ are binary variables\\
			Monotonicity ~\citep{huitfeldt2018choice,cinelli2021generalizing}& $Y(1)\geq Y(0)$\\ \midrule
			Invariance of counterfactual expectation transition mechanisms (ours)& ${\mathbb E}[Y(1)|Y(0)=y,X=x,D=i]={\mathbb E}[Y(1)|Y(0)=y,X=x,D=0]$\\
			Conditional Mixture Distribution (ours)& $\Pr(Y(0)\leq y|X=x,D=0)=\sum_{i=1}^{N}w_i^*(x)\Pr(Y(0)\leq y|X=x,D=i)$
			\\ \bottomrule
	\end{tabular}}
	\vspace{-13pt}
	\label{table:summary of the identification assumptions}
\end{table}
\section{Weight Estimation by Minimizing the CMMD}\label{section: Weight estimation by minimizing the CMMD}


Under Assumption \ref{assumption: the invariance of counterfactual expectation transition mechanism}, if we can estimate the CETM ${\mathbb E}[Y(1)|Y(0)=y,X=x,D=0]$, then we can directly estimate the treatment group expectation by reweighting since 
\begin{align*}
	{\mathbb E}[Y(1)|X=x,D=0]={\mathbb E}\left\{{\mathbb E}[Y(1)|Y(0),X,D=0]|X=x,D=0\right\}.
\end{align*}
However,  we can't observe both $Y(1)$ and $Y(0)$ at the same time which is also known as the fundamental problem of causal inference \citep{Holland1985StatisticsAC}. As a result, we can not directly estimate ${\mathbb E}[Y(1)|Y(0)=y,X=x]$ without further assumptions. According to Theorem \ref{proposition: identification}, we can circumvent this problem by estimating the weights $\{w_i^*(x))\}_{i=1}^N$ to create a synthetic treatment group. Instead of assigning equal weights to each source population, the weights are chosen in a more principled way to reflect the relevance of the source populations with the target population.

The synthetic control method estimates the weights by matching the first order moment of the weighted control unit with the treated unit in the pre-intervention periods \citep{wong2015three}. Without using the higher order moments, this method may be problematic if the number of populations $N$ is large and there is only one pre-treatment period. In order to leverage the higher order information of the distributions, an intuitive approach will be to choose a vector-valued function $ g(y) $ of dimension $ d $ ($ d\geq N $) and use the generalized method
of moments \citep{hall2004generalized} to solve the following minimization problem:\footnote{Here we denote $\sum_{i=1}^{N}w_i(x) P_{Y(0)|X=x,D=i}$ as the mixture distribution with the cumulative distribution function $\sum_{i=1}^{N}w_i(x) \Pr(Y(0) \leq y|X=x,D=i)$.}
\begin{align*} 
		\underset{\{w_i(x))\}_{i=1}^N}{\arg\min}~ \| {\mathbb E}[g(Y)]-{\mathbb E}[g(Z)] \| ,
		\mbox{ where } Y\sim\sum_{i=1}^{N}w_i(x)P_{Y(0)|X=x,D=i} \mbox{ and }	Z\sim P_{Y(0)|X=x,D=0}.
\end{align*}
However, it is difficult to choose the moment function $ g(y) $ in practice without sufficient background knowledge. We propose a method to choose $g(y)$ automatically. Before presenting it, we need to introduce the concepts of the conditional mean embedding and the conditional maximum mean discrepancy.

\subsection{Conditional Mean Embedding and Conditional Maximum Mean Discrepancy}

We refer to $X\colon\Omega \to \mathcal{X}$ and $Y\colon\Omega \to \mathcal{Y}$ and $Z\colon\Omega \to \mathcal{Y}$ as the random variables defined on the probability space $(\Omega,\mathscr{F},\mathbb{P})$. Denote $\mathcal{H}$ and $\mathcal{G}$ as the reproducing kernel Hilbert space (RKHS) induced by the symmetric positive definite kernels $ k\colon\mathcal{X}\times\mathcal{X}\to \mathbb{R} $ and $ \ell\colon \mathcal{Y}\times\mathcal{Y}\to \mathbb{R} $, respectively. For every $x\in \mathcal{X}$ and $y\in \mathcal{Y}$, the canonical feature mappings are defined as $k(x,\cdot)\in\mathcal{H}$ and $\ell(y,\cdot)\in\mathcal{G}$. Through the conditional mean embedding, we can represent a probability distribution as an element in the RKHS. The conditional mean embedding (CME) of the conditional distribution of Y given X in the RKHS $ \mathcal{G} $ with respect to the kernel $ \ell $ is defined as the following Bochner integral \citep{song2009hilbert,muandet2017kernel}: 
\begin{equation*}
	\mu_{Y|X=x}={\mathbb E}[\ell(Y,\cdot)|X=x].    
\end{equation*}
With the CME, we can define a probability metric between conditional distributions $P_{Y|X=x}$ and $P_{Z|X=x}$, which is called the conditional maximum mean discrepancy (CMMD) \citep{Ren2016ConditionalGM,park2021conditional}:
\[
\Vert \mu_{Y|X=x}-\mu_{Z|X=x}\Vert^2_\mathcal{G},	\mbox{ where } \Vert \cdot \Vert_{\mathcal{G}} \mbox{ denotes the norm of } \mathcal{G}.\]
The CMMD characterizes the difference between two conditional distributions. It equals zero if and only if two conditional distributions are identical to each other almost surely \citep{park2021conditional}. In the spirit of the minimum distance estimator \citep{cherief2022finite}, we estimate the weights by minimizing the CMMD between the weighted control group distribution of the source populations and the control group distribution of the target population:
\begin{align}\label{equation:Integral probability metric}
	\{w^*_i(x)\}_{i=1}^N=\underset{\{w_i(x))\}_{i=1}^N}{\arg\min}~d(x,w(x))=\underset{\{w_i(x))\}_{i=1}^N}{\arg\min}~\sup_{\Vert f\Vert_{\mathcal{G}}\leq1}\left\vert {\mathbb E}[f(Y)]-{\mathbb E}[f(Z)]\right\vert,
\end{align}
where $Y\sim \sum_{i=1}^{N}w_i(x)P_{Y(0)|X=x,D=i}$ and $Z\sim P_{Y(0)|X=x,D=0}$. Here we denote the CMMD as $d(x,w(x))=\Vert \mu_{\sum_{i=1}^{N}w_i(x)P_{Y(0)|X=x,D=i}}-\mu_{P_{Y(0)|X=x,D=0}}\Vert_{\mathcal{G}}^2$ and $w(x)=[w_1(x), \ldots, w_N(x)]^T$.

Eqn.~\ref{equation:Integral probability metric} follows from \citep{park2021conditional}. The method of minimizing the CMMD can be regarded as a way to select the moment function $f(y)$ automatically in the unit ball of the RKHS $\mathcal{G}$ according to Eqn.~\ref{equation:Integral probability metric}.
In order to estimate the CMMD $d(x,w(x))$, we need to first estimate the CME.

\subsection{Estimation of Conditional Mean Embedding}

According to the SUTVA assumption, we refer to $\{Y_{i,j}(0),X_{i,j}\}_{j=1}^{n_i}$ and $\{Y_{0,j}(0),X_{0,j}\}_{j=1}^{n_0}$ as the control group data from the $i$-th source population and the target population. 

We can estimate the CME of the control group distribution of the $i$-th source population ($i=0$ refers to the target population) by using the sample from the control group as \citep{song2009hilbert,muandet2017kernel}$\colon$ 
\begin{equation}\label{equation: notations}
	\hat{\mu}_{Y(0)|X=x,D=i}= \ell_i^TM_ik_i(x),
\end{equation}
where $k_i(x)=(k(X_{i,1},x),\cdots, k(X_{i,n_i},x))^T$, $M_i=(K_i+\lambda_i I_{n_i})^{-1}$, $K_i$ denotes the $n_i\times n_i$ kernel matrix with $\left[K_i\right]_{u,v}=k(X_{i,u},X_{i,v})$,  $\lambda_i$ denotes the regularization hyperparameter and $\ell_i=(\ell(Y_{i,1}(0), \cdot), \ldots, \ell(Y_{i, n_i}(0), \cdot))^T$.

One advantage of the CME estimator is that we don't need to specify a parametric model for the conditional distribution. It circumvents the problem of model misspecification \citep{muandet2017kernel}. By the linearity of integral operation, we have
\begin{align*}
	\mu_{\sum_{i=1}^{N}w_i(x)P_{Y(0)|X=x,D=i}}=\sum_{i=1}^{N}w_i(x)\mu_{Y(0)|X=x,D=i}.
\end{align*}
It follows that the CME of the mixture distribution can be estimated as
\begin{align*}
	\hat{\mu}_{\sum_{i=1}^{N}w_i(x) P_{Y(0)|X=x,D=i}}=\sum_{i=1}^{N}w_i(x)\hat{\mu}_{Y(0)|X=x,D=i}=\begin{bmatrix}
		\ell_1^T M_1 k_1(x),
		\ldots,
		\ell_N^TM_Nk_N(x)
	\end{bmatrix}w(x).
\end{align*}

\subsection{Pointwise Weight Estimation}

We can substitute the CME estimators into Eqn.~\ref{equation:Integral probability metric} to obtain the CMMD estimator. 
\begin{theorem}
	\label{proposition: estimation} 
	The CMMD estimator for $d(x,w(x))$ is
	\begin{equation}\label{equation: CMMD calculation}
		\hat{d}(x,w(x)) = w(x)^T\hat{A}(x)w(x)-2w(x)^T\hat{b}(x)+k_0^T(x)M_0\ell_0\ell_0^TM_0k_0(x).    
	\end{equation}
	It is a convex quadratic function with respect to $w(x)$. Futhermore, $\hat{d}(x,w(x))$ is minimized when $w(x)=\hat{A}(x)^{-1}\hat{b}(x)$.
\end{theorem}
We follow the notation from Eqn.~\ref{equation: notations}. Here $\hat{A}(x)$ denotes an $N\times N$ matrix with \\$[\hat{A}(x)]_{i,j}=k_i^T(x)M_i \ell_i \ell_j^T M_j k_j(x)$. It is an estimator for matrix $A(x)$ with $[A(x)]_{i,j}=\langle\mu_{Y(0)|X=x,D=i},\mu_{Y(0)|X=x,D=j}\rangle_\mathcal{G}.$ By the reproducing kernel property, $\ell_i \ell_j^T$ is a $n_i\times n_j$ matrix with $[\ell_i \ell_j^T]_{u,v}=\left\langle \ell(Y_{i,u}(0), \cdot), \ell(Y_{j,v}(0), \cdot)\right\rangle_\mathcal{G}=\ell(Y_{i,u}(0),Y_{j,v}(0))$. And $\hat{b}(x)$ is an $N$-dimensional vector with $\left[\hat{b}(x)\right]_i=k_i^T(x)M_i \ell_i \ell_0^T M_0 k_0(x)$. It is an estimator for vector $b(x)$ with $\left[b(x)\right]_i=\langle\mu_{Y(0)|X=x,D=i},\mu_{Y(0)|X=x,D=0}\rangle_\mathcal{G}$.

We can estimate the CMMD without explicitly modeling the conditional density function which is a difficult task especially in high dimension \citep{muandet2017kernel}. When we place the constraint on the weights such that $\sum_{i=1}^{N}w_i(x)=1$ and $\underset{i}{\min} \:w_i(x)\geq 0$, the weight estimation is a constrained quadratic programming problem. The pointwise estimation of $\{w_i(x)\}_{i=1}^N$ can be carried out for different $x$ in parallel. The pointwise weight estimation procedure is similar to the mixture proportion estimation \citep{hall1981non,iyer2014maximum,iyer2016privacy,yu2018efficient}. With access to samples from the target distribution and each component distribution, they estimate the finite dimensional weight vector by minimizing the maximum mean discrepancy between the weighted component distribution and the target distribution. In contrast, we additionally include the covariate information. As a result, we don't have access to samples directly from the component distributions of the control groups when the covariate variable is continuous. Consequently, we face a more difficult challenge of estimating several infinite dimensional weight functions.

\section{Sieve Semiparametric Two-step Estimation}\label{section: Sieve semiparametric two-step estimation}

The pointwise weight estimation method enjoys the advantage of being intuitive and general. It does not rely on any parametric modeling assumptions. However, in some situations, the pointwise weight estimation method may result in wiggly estimates (see Figure \ref{fig:weight estimation}). It may run into the risk of overfitting the data \citep{anderson2002model}. In many cases, it is usually plausible to presume that individuals with similar covariates $x$ will tend to have similar weights $w_i(x)$. In contrast, a complex and wiggly weight estimator may contradict with this prior knowledge and lose interpretability. In order to overcome this problem, we leverage the smoothness property by approximating the weight function $w_i(\cdot)$ via a linear sieve space $\mathcal{W}_i^{s(n_T)} = \{w_i(\cdot)=P_i(\cdot)^T\beta_{w_i}|\beta_{w_i}\in {\sR}^{s(n_T)}\}$, where $P_i(x)=\left(p_{i,1}(x), \ldots, p_{i,s(n_T)}(x)\right)^T$ is a set of known basis function such as the B-spline basis and $s(n_T)\to \infty$ slowly as $n_T\to \infty$ \citep{chen2007large}. Here we denote $n_T=n_0+\sum_{i=1}^{N}(m_i+n_i)$ as the total sample size. We can represent weights as $w(x)=V(x)^T\beta_w$, where $V(x)=\mathrm{diag}(P_1(x),\cdots,P_N(x))$ and $\beta_w=\left[\beta_{w_1}^T, \ldots, \beta_{w_N}^T\right]^T$. We can estimate the sieve coefficients $\beta_w$ by minimizing the average CMMD of the target population. Substituting $w(x)=V(x)^T\beta_w$ into Eqn.~\ref{equation: CMMD calculation}, the estimator for ${\mathbb E}\left[d(X,w(X))|D=0\right]$ can be represented as
\begin{align}\label{equation: minimizing average cmmd}
	\frac{1}{n_0}\sum_{i=1}^{n_0}\hat{d}(x_{0,i},w(x_{0,i}))=\:&\frac{1}{n_0}\sum_{i=1}^{n_0}\beta_w^TV(x_{0,i})\hat{A}(x_{0,i})V(x_{0,i})^T \beta_w-\frac{2}{n_0}\sum_{i=1}^{n_0}\left(V(x_{0,i})\hat{b}(x_{0,i})\right)^T\beta_w\notag\\
	&+\frac{1}{n_0}\sum_{i=1}^{n_0}k_0^T(x_{0,i})M_0 \ell_0 \ell_0^TM_0k_0(x_{0,i}).
\end{align}
Similar to Theorem \ref{proposition: estimation}, the sieve coefficient $\beta_w$ can be estimated by the quadratic programming algorithm. Given the sieve coefficient estimator $\hat{\beta}_{w_i}\in {\sR}^{s(n_T)}$, the weight estimator is $\hat{w}_i(x) = (\hat{\beta}_{w_i})^TP_i(x).$	

Analogously, we can approximate the outcome regression function $g^*_i(\cdot)$ via a linear sieve space $\mathcal{Q}_i^{t(n_T)} = \{g_i(\cdot)=Q_i(\cdot)^T\beta_{g_i}|\beta_{g_i}\in {\sR}^{t(n_T)}\}$, where $Q_i(x)=\left(q_{i,1}(x), \ldots, q_{i, t(n_T)}(x)\right)^T$ is a set of known basis function and $t(n_T)\to \infty$ as $n_T\to \infty$ \citep{chen2007large}. We denote by $\{y_{i,j}',x_{i,j}'\}_{j=1}^{m_i}$ the treatment group data from the $i$-th source population. The sieve coefficient $\hat{\beta}_{g_i}$ can be estimated as 
\begin{equation}\label{equation: sieve outcome regression}
	\hat{\beta}_{g_i}=\underset{\beta_{g_i}\in {\sR}^{t(n_T)}}{\arg \min}~\frac{1}{m_i}\sum_{j=1}^{m_i}(y'_{i,j}-\beta_{g_i}^TQ_i(x'_{i,j}))^2, \mbox{ for } i=1,\ldots,N.   
\end{equation}
According to Theorem \ref{proposition: identification}, we can marginalize over the covariate and construct the synthetic treatment group estimator for $\theta_0={\mathbb E}[Y(1)|D=0]$ as
\begin{equation}\label{equation: treatment group expectation estimator}
	\hat{\theta}=\frac{1}{n_0}\sum_{j=1}^{n_0}\sum_{i=1}^{N}\hat{w}_i(x_{0,j})\hat{g}_i(x_{0,j}),
\end{equation}
where $\hat{w}_i(x) = \hat{\beta}_{w_i}^TP_i(x)$ denotes the sieve extremum estimator for the weight function $w^*_i(x)$, $\hat{g}_i(x)=\hat{\beta}_{g_i}^TQ_i(x)$ denotes the sieve extremum estimator for the outcome regression function $g^*_i(x)$ and $\{x_{0,j}\}_{j=1}^{n_0}$ denotes the covariate data from the target population. Note that even when Assumptions \ref{assumption: the invariance of counterfactual expectation transition mechanism} and \ref{asumption:mixture distribution} are violated, Eqn.~\ref{equation: treatment group expectation estimator} is still a consistent estimator as long as the mean exchangeability assumption holds according to Proposition \ref{proposition: generalize the identification formula under the mean exchangeability assumption}. 

To the best of our knowledge, it is extremely difficult to derive the asymptotic distribution of the weighted estimator based on the pointwise weight estimator, which is similar to the problem of kernel conditional discrepancy \citep{park2021conditional}. This motivates us to apply the sieve semiparametric two-step estimation method \citep{chen2015sieve}.
\subsection{Two-step Estimation} 

Note that the synthetic treatment group estimator $\hat{\theta}$ can be formulated as a semiparametric two-step GMM estimation method where the first step consists of nonparametric estimation of weights and outcome regression functions \citep{chen2015sieve}. Let $Z=(D, A, X, Y).$ The true weights $\{w^*_i(\cdot)\}_{i=1}^N$ and outcome regression functions $\{g^*_i(\cdot)\}_{i=1}^N$ form the solution $h_0$ to the following infinite dimensional optimization problem:
\begin{align*}
	\underset{h}{\arg \sup}~Q(h)=\underset{h}{\arg \sup}~{\mathbb E}[\varphi(Z,h)]=\underset{h}{\arg \sup}~ \sE\left[\varphi_w(Z,w(X))+\sum_{i=1}^{N}\varphi_{g_i}(Z,g_i)\right],
\end{align*}
${\mbox{where }Q(h)=-\sE\left[d(X,w(X))|D=0\right]-\sum_{i=1}^{N}{\mathbb E}[(Y-g_i(X))^2|A=1,D=i]}$\\
and ${h=[w_1,\cdots,w_N,g_1,\cdots,g_N]}$. The moment function can be decomposed into two parts as $\varphi(Z,h)=\varphi_w(Z, w(X))+\sum_{i=1}^{N}\varphi_{g_i}(Z,g_i)$. Here we let $\varphi_w(Z,w(X))=-\frac{I_{D=0}}{\Pr(D=0)}d(X,w(X))$ and $\varphi_{g_i}(Z,g_i)=-\frac{I_{A=1,D=i}}{\Pr(A=1,D=i)}(Y-g_i(X))^2$.

Combining Eqns. \ref{equation: minimizing average cmmd} and \ref{equation: sieve outcome regression}, we employ the first step sieve extremum estimation to estimate the weights and outcome regression functions by minimizing the sample analog of the criterion function $Q(h)$:
\begin{align*}
	\hat{h}=\underset{h\in \mathcal{S}_{n_T}}{\arg \sup}~\hat{Q}(h)=\underset{h\in \mathcal{S}_{n_T}}{\arg \sup}~-\frac{1}{n_0}\sum_{j=1}^{n_0}\hat{d}(x_{0,j},w(x_{0,j}))-\sum_{i=1}^{N}\frac{1}{m_i}\sum_{j=1}^{m_i}(y'_{i,j}-g_i(x'_{i,j}))^2,
\end{align*}
where $\mathcal{S}_{n_T}=\{h\colon w_i(\cdot)\in \mathcal{W}_i^{s(n_T)} \mbox{ and } g_i(\cdot)\in \mathcal{Q}_i^{t(n_T)} \mbox{ for }i=1, \ldots, N\}$ is the linear sieve space.

The moment function of the second step GMM estimation is defined as
\begin{equation*}
	g(Z,\theta,h)=\theta-\frac{I_{D=0}}{\Pr(D=0)}\sum_{i=1}^{N}w_i(X)g_i(X),
\end{equation*}
where the parameter of interest is $\theta_0={\mathbb E}[Y(1)|D=0]$. Given the first step estimator $\hat{h}$, the second step GMM estimator of $\theta_0$ is equivalent to the synthetic treatment group estimator in Eqn.~\ref{equation: treatment group expectation estimator}.

\subsection{Statistical Inference}

We can establish the asymptotic normality of the synthetic treatment group estimator by characterizing the influence of the first step estimation \citep{chen2015sieve}. 

\begin{theorem}
	\label{theorem: asymptotic normality}
	Assume that $\{Z_i\}_{i=1}^{n_T}$ are independent and identically distributed. Under Assumptions \ref{assumption: the invariance of counterfactual expectation transition mechanism}, \ref{asumption:mixture distribution} and Assumptions A1 and A2 in Appendix \ref{subsection: proof of theorem 5}, we have
	\begin{align*}
		\sqrt{n_T}(\hat{\theta}-\theta_0)\overset{d}{\to} N(0,V_{\theta}),
	\end{align*}
	where $V_{\theta}=(\Gamma_1'W\Gamma_1)^{-1}(\Gamma_1'WV_1W\Gamma_1)(\Gamma_1'W\Gamma_1)^{-1}=V_1 = \Var\left(g(Z_i,\theta_0,h_0)+\Delta(Z_i,h_0)[v^*]\right)$ due to the fact that $\Gamma_1=\frac{\partial {\mathbb E}[g(Z,\theta,h)]}{\partial \theta}=1$. The adjustment term is defined as $\Delta(Z,h)[v]=\frac{\partial \varphi(Z,h+\tau v)}{\partial \tau}|_{\tau=0}$, which reflects the influence of the first step estimation, and $v^*$ denotes the Riesz representer of the Gateaux derivative $\Gamma_2(\theta_0,h_0)[v]=\frac{\partial {\mathbb E}[g(Z,\theta_0,h_0+\tau v)]}{\partial \tau}|_{\tau=0}$.
\end{theorem}
We can estimate the asymptotic variance $V_{\theta}$ by approaching the problem based on the parametric belief \citep{ackerberg2012practical}. \citet{chen2015sieve} and \citet{ackerberg2012practical} establish the numerical equivalence result of the asymptotic variance estimator based on the parametric belief and the corresponding estimator under the nonparametric specification.\footnote{Please refer to Section 5 of \citet{chen2015sieve} for more details.} Under the parametric belief, we assume that the true values of the unknown parameters lie in the linear sieve space such that $w^*_i(\cdot)={\beta_{o,w_i}}^TP_i(\cdot)$ and $g^*_i(\cdot)={\beta_{o,g_i}}^TQ_i(\cdot),$ for $i=1,\ldots,N$. Representing the nuisance parameters by their corresponding sieve coefficients $\beta_w$ and $\beta_g$, we can reformulate the criterion functions by adding the subscript $P$ as 
\begin{align*}
	&\varphi_{w,P}(Z,\beta_{w})=\varphi_{w}(Z,V(X)^T\beta_w),\\
	&\varphi_{g_i,P}(Z,\beta_{g_i})=\varphi_{g_i}(Z,{\beta_{g_i}}^TQ_i(\cdot)) \mbox{ for } i=1,\ldots,N.
\end{align*}

Let $\beta=[\beta_w^T,\beta_g^T]^T$ denote the sieve coefficients and $\beta_{o,P}=[\beta_{o,w}^T,\beta_{o,g}^T]^T$ denote their corresponding true values. The criterion function under the parametric belief is 
\begin{align*}
	\varphi_P(Z,\beta)=\varphi_{w,P}(Z,\beta_{w})+\sum_{i=1}^N \varphi_{g_i,P}(Z,\beta_{g_i}).   
\end{align*}
The second step moment function can be reformulated as 
\begin{align*}
	g_P(Z,\theta,\beta_{w},\beta_{g})=g\left(Z,\theta,\{{\beta_{w_i}}^TP_i(\cdot)\}_{i=1}^N,\{{\beta_{g_i}}^TQ_i(\cdot)\}_{i=1}^N\right).    
\end{align*}
The adjustment term in Theorem \ref{theorem: asymptotic normality} based on parametric belief can be decomposed into two parts as
\begin{equation}\label{equation: adjustment term based on the parametric belief}
	\Delta_P(Z,\beta_{o,w},\beta_{o,g})[v_P^*]=\Gamma_{2,w}(R_w)^{-1}\frac{\partial \varphi_{w,P}(Z,\beta_{o,w})}{\partial \beta_w}+\sum_{i=1}^{N}\Gamma_{2,g_i}(R_{g_i})^{-1}\frac{\partial \varphi_{g_i,P}(Z,\beta_{o,g_i})}{\partial \beta_{g_i}}.
\end{equation} 
The first part with the subscript $w$ characterizes the influence of the weight estimation. Here we let 
\begin{align*}
	&R_w=-\mathbb{E}[\frac{\partial^2 \varphi_{w,P}(Z,\beta_{o,w})}{\partial \beta_w\partial \beta_w^T}]=2\mathbb{E}[V(X)A(X)V(X)^T|D=0],\\
	&\Gamma_{2,w}=\frac{\partial \mathbb{E}[g(Z,\theta_0,\beta_{o,w},\beta_{o,g})]}{\partial \beta_w^T}=-E\left\{[g^*_1(X)P_1^T(X),\cdots,g^*_N(X)P_N^T(X)]|D=0\right\},\\
	&\frac{\partial \varphi_{w,P}(Z,\beta_{o,w})}{\partial \beta_w}=\frac{I_{D=0}}{P(D=0)}\left\{-2V(X)A(X)V(X)^T\beta_{o,w}+2V(X)b(X)\right\}.
\end{align*}
Recall that $V(x)=\mathrm{diag}(P_1(x),\cdots,P_N(x))$ and $P_i(x)$ is the basis function for estimating the weight $w_i^*(x)$ of the $i$-th source population.

The second part with the subscript $g$ characterizes the influence of the outcome regression estimation. Here we let
\begin{align*}
	&R_{g_i}=-\mathbb{E}[\frac{\partial^2 \varphi_{g_i,P}(Z,\beta_{o,g_i})}{\partial \beta_{g_i}\partial \beta_{g_i}^T}]=2\mathbb{E}[Q_i(X)Q_i(X)^T|A=1,D=i],\\
	&\Gamma_{2,g_i}=\frac{\partial \mathbb{E}[g(Z,\theta_0,\beta_{o,w},\beta_{o,g})]}{\partial \beta_{g_i}^T}=-\mathbb{E}[w^*_i(X)Q_i(X)^T|D=0],\\
	&\frac{\partial \varphi_{g_i,P}(Z,\beta_{o,g_i})}{\partial \beta_{g_i}}=\frac{I_{A=1,D=i}}{P(A=1,D=i)}\{-2Q_i(X)Q_i(X)^T\beta_{o,g_i}+2Q_i(X)Y\}.
\end{align*}  
Recall that $Q_i(x)$ is the basis function for estimating the outcome regression function $g_i^*(x)$.

In order to estimate the adjustment term, we need to consider its sample analog 
\begin{align*}
	\hat{\Delta}_P(Z,\hat{\beta}_{w},\hat{\beta}_{g})[\hat{v}_P]=\hat{\Gamma}_{2,w}(\hat{R}_w)^{-1}\frac{\partial \hat{\varphi}_{w,P}(Z,\hat{\beta}_w)}{\partial \beta_w}+\sum_{i=1}^{N}\hat{\Gamma}_{2,g_i}(\hat{R}_{g_i})^{-1}\frac{\partial \varphi_{g_i,P}(Z,\hat{\beta}_{g_i})}{\partial \beta_{g_i}}.   
\end{align*}

Note that the sample analog of the adjustment term can also be decomposed into two parts. For the first part, we let
\begin{align*}
	&\hat{\Gamma}_{2,w}=-\frac{1}{n_0}\sum_{i=1}^{n_0}[\hat{g}_1(x_{0,i})P_1^T(x_{0,i}),\cdots,\hat{g}_N(x_{0,i})P_N^T(x_{0,i})],\\
	&\hat{R}_w=\frac{2}{n_0}\sum_{i=1}^{n_0}V(x_{0,i})\hat{A}(x_{0,i})V^T(x_{0,i}),\\
	&\frac{\partial \hat{\varphi}_{w,P}(Z,\hat{\beta}_w)}{\partial \beta_w}=\frac{I_{D=0}}{P(D=0)}\left\{-2V(X)\hat{A}(X)V(X)^T\hat{\beta}_w+2V(X)\hat{b}(X)\right\}.	
\end{align*}

Recall that $\{x_{0,j}\}_{j=1}^{n_0}$ denotes the covariate data from the target population.

For the second part, we let
\begin{align*}
	&\hat{\Gamma}_{2,g_i}=\frac{1}{m_i}\sum_{j=1}^{m_i}\hat{w}_i(x_{i,j}')Q_i(x_{i,j}')^T,\\
	&\hat{R}_{g_i}=\frac{2}{m_i}\sum_{j=1}^{m_i}Q_i(x_{i,j}')Q_i(x_{i,j}')^T,\\
	&\frac{\partial \varphi_{g_i,P}(Z,\hat{\beta}_{g_i})}{\partial \beta_{g_i}}=\frac{I_{A=1,D=i}}{P(A=1,D=i)}\{-2Q_i(X)Q_i(X)^T\hat{\beta}_{g_i}+2Q_i(X)Y\}.
\end{align*}

Recall that $\{y_{i,j}',x_{i,j}'\}_{j=1}^{m_i}$ denotes the treatment group data from the $i$-th source population.

Given the sample analog of the adjustment term, we can define the sieve score estimator as 
\begin{align*}
	\hat{S}_{i,P}=g_P(Z_i,\hat{\theta},\hat{\beta}_{w},\hat{\beta}_{g})+\hat{\Delta}_P(Z_i,\hat{\beta}_{w},\hat{\beta}_{g})[\hat{v}_P],    
\end{align*}
where the second step moment function based on the parametric belief is defined as
\begin{align*}
g_P(Z,\hat{\theta},\hat{\beta}_{w},\hat{\beta}_{g})=g\left(Z,\hat{\theta},\{{\hat{\beta}_{w_i}}^TP_i(\cdot)\}_{i=1}^N,\{{\hat{\beta}_{g_i}}^TQ_i(\cdot)\}_{i=1}^N\right).
\end{align*}

Based on the sieve score estimator, we can construct the asymptotic variance estimator as 
\begin{align*}
\hat{V}_{\theta}=\frac{1}{n_T}\sum_{i=1}^{n_T}(\hat{S}_{i,P}-\frac{1}{n_T}\sum_{j=1}^{n_T}\hat{S}_{j,P})^2.
\end{align*}

\begin{proposition}
	\label{proposition: asymptotic variance estimator}
	Under assumptions of Theorem \ref{theorem: asymptotic normality} and Assumptions A3 and A4 in Appendix \ref{subsection: proof of proposition 6}, $\hat{V}_{\theta}$ is a consistent asymptotic variance estimator such that $\hat{V}_{\theta}\overset{p}{\to}V_{\theta}$.
	
 Based on the fact that
	\begin{align*}
		\sqrt{n_T}\hat{V}_{\theta}^{-\frac{1}{2}}(\hat{\theta}-\theta_0)\overset{d}{\to} N(0,1),
	\end{align*}
we can construct the asymptotic confidence interval with the coverage rate $1-\alpha$ for $ \theta_0$ as 
\begin{align}\label{equation: confidence interval}
	[\hat{\theta}-\frac{\hat{V}_{\theta}^{\frac{1}{2}}}{\sqrt{n_T}}z_{1-\frac{\alpha}{2}},\hat{\theta}+\frac{\hat{V}_{\theta}^{\frac{1}{2}}}{\sqrt{n_T}}z_{1-\frac{\alpha}{2}}],    
\end{align}
where $z_{1-\frac{\alpha}{2}}$ is the $1-\frac{\alpha}{2}$ quantile of the standard normal distribution.
\end{proposition}

\section{Experimental Results}

We conduct both the simulated experiment and the real data experiment to compare the performance of our method with the classical transportability method.

\subsection{Simulated Experiment}\label{section: Simulated experiment}

In order to further support Examples \ref{example: mixture distribution} and \ref{example: violation of mean exchangeability}, the simulated experiment is designed according to their settings. For the simulated experiment, we consider three source populations. We generate the control group potential outcome of the source populations from the mixture normal distributions. Like in Example \ref{example: mixture distribution}, the control group distributions of the source populations share the same mixture components but differ in the mixture weights: 
\begin{align*}
	Y(0)|D=i,X=x\sim\sum_{j=1}^{3}\pi_{i,j}(x)N(a_jx+b_j,1).	    
\end{align*}

We generate the control group potential outcome of the target population by the weighted mixture distribution of the source populations:
\begin{align*}
	P_{Y(0)|X=x,D=0}=\sum_{i=1}^{N}w^*_i(x)P_{Y(0)|X=x,D=i}.
\end{align*}

The weights $\{w^*_i(x)\}_{i=1}^N$ and $\{\pi_{i,j}(x)\}_{j=1}^3$ obey the multinomial linear logistic model such that $w^*_i(x)=\frac{exp(c_ix+d_i)}{\sum_{j=1}^{N}exp(c_jx+d_j)}$ and $\pi_{i,j}(x)=\frac{exp(e_{i,j}x+f_{i,j})}{\sum_{k=1}^{3}exp(e_{i,k}x+f_{i,k})}$.

In accordance with Example \ref{example: violation of mean exchangeability}, for the treatment group of the source populations, we first generate the covariate and the unobserved intermediate control group counterfactual outcome. Then we generate the treatment group potential outcome from the intermediate control group counterfactual outcome through the counterfactual expectation transition mechanism. For the $i$-th population, we generate the treatment group potential outcome for an individual with the intermediate control group counterfactual outcome $y$ and the covariate $x$ as 
\begin{align*} 
		Y(1)&=\mathbb{E}[Y(1)|Y(0)=y,X=x,D=0]+\epsilon=g_1 y+g_2 x+g_3 xy+\epsilon,
\end{align*} 
where $\mathbb{E}[\epsilon|Y(0),X,D=i]=0$.

In order to evaluate the performance under a wide range of experimental conditions, we generate the parameters of the data generating process from the normal distribution for each round of the experiments instead of restricting the evaluation to a fixed set of parameters. Please refer to Appendix \ref{subsection: simulated experiments} for a detailed description of the data generating process.     

We implement our proposed two methods \textbf{Sieve} and \textbf{Point}, with comparison with the baseline methods \textbf{Pool}~\citep{dahabreh2020towards} and \textbf{Uniform}.
\begin{description}[nosep]
	\item[\textbf{Sieve }]  We apply the sieve extremum estimation method to estimate the weights according to Eqn.~\ref{equation: minimizing average cmmd}.
	\item[\textbf{Point }]  We employ pointwise weight estimation methods in Theorem \ref{proposition: estimation} both with and without the constraint that the weight estimators are non-negative and add up to one. 
	\item[\textbf{Pool }]  Pool the data from the source populations and estimate the outcome regression function. Then use the covariate information from the target population to estimate $\theta_0$. It corresponds to the classical transportability method based on the mean exchangeability assumption.
	\item[\textbf{Uniform }]  Use the uniform weights such that $w_i(x)\equiv\frac{1}{N}$ for any $i$ and $x$.
\end{description}

The estimation performance was measured by the mean relative error (MRE):
\begin{align*}
	\mathrm{MRE}\triangleq\frac{1}{M}\sum_{i=1}^M\left\vert\frac{\hat{\theta}_i-\hat{\theta}_{0,i}}{\hat{\theta}_{0,i}}\right\vert.
\end{align*}
Here we refer to $\hat{\theta}_i$ as the synthetic treatment estimator in Eqn.~\ref{equation: treatment group expectation estimator} of the $i$-th round of the experiments and $M$ denotes the total number of iterations. We refer to $\hat{\theta}_{0,i}$ as the sample average estimator for the treatment group expectation $\theta_0=E[Y(1)|D=0]$ which serves as the ground truth. It is estimated by the unobserved treatment group data of the target population.

The mean relative error of the simulated experiment is summarized in Table \ref{table: MAR simulated experiments} with the standard errors in the brackets. Recall that the label "Sieve" refers to the synthetic treatment group esimator calculated by Eqn.~\ref{equation: treatment group expectation estimator} and the label "Constrained" refers to the estimator based on the pointwise weight estimator with the constraint that the weights are nonnegative and adds up to one. The label "Sample size" refers to the sample size of the control group of each source population and the target population. As shown in Table \ref{table: MAR simulated experiments}, our proposed methods (with labels as "Sieve", "Constrained" and "Unconstrained") outperform two baseline methods (with labels as "Uniform" and "Pool") in all situations. In consistent with the theoretical result in Example \ref{example: violation of mean exchangeability}, the baseline methods fail to identify the treatment group expectation. 

\begin{table}[!ht]
	\centering
	\caption{Mean Relative Error of the Simulated Experiments.}\label{table: MAR simulated experiments}
	\begin{tabular}{llllll}
		\hline
		Sample size & Sieve & Constrained & Unconstrained & Uniform & Pool \\ \hline
		4000 & 0.042(0.062) & 0.043(0.063) & 0.042(0.058) & 0.879(1.238) & 0.881(1.242) \\ 
		3000 & 0.058(0.109) & 0.059(0.111) & 0.062(0.115) & 0.999(1.833) & 1.002(1.822) \\ 
		2000 & 0.09(0.236) & 0.094(0.249) & 0.088(0.233) & 1.987(6.956) & 1.978(6.927) \\ \hline
	\end{tabular}
\end{table}

Figure \ref{fig:weight estimation} demonstrates the performance of weight estimation for the first source population where the label "source" refers to the true value. The pointwise weight estimator oscillates drastically around the true value. In contrast, the weight estimator based on the sieve extremum estimation closely approximates the true value. Compared with the pointwise weight estimator, the sieve extremum estimation alleviates the problem of overfitting. 

The CMMD between two source populations can reflect the difficulty of the weight estimation. Figure \ref{fig: conditional maximum mean discrepancy between weighted source populations and the target population.} depicts the CMMD between the weighted source populations and the target population calculated by Eqn.~\ref{equation: CMMD calculation}. When we use the proposed method, the CMMD between the weighted source populations and the target population is near zero. By contrast, the CMMD based on the uniform weights is much larger as shown by the red curve in Figure \ref{fig: conditional maximum mean discrepancy between weighted source populations and the target population.}. This indicates that the simple average of the source populations can't approximate the target population well. It confirms with the poor performance of the baseline method based on the uniform weights. The additional experimental results in Appendix \ref{subsection: simulated experiments} demonstrate that our proposed methods consistently outperform the baseline methods.
\begin{figure}[H]
    \centering
    \includegraphics[width=0.5\textwidth]{"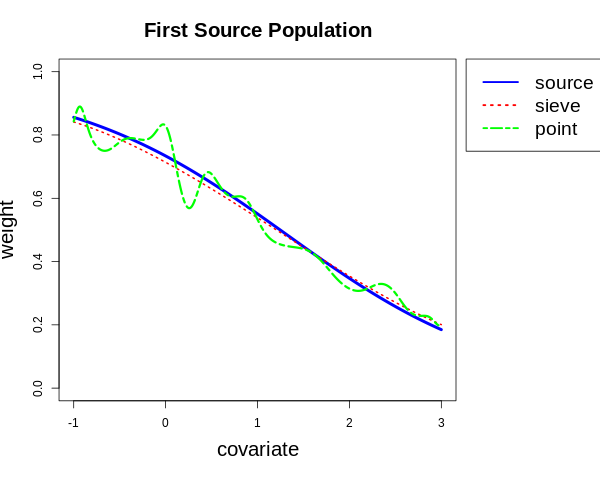"}
    \caption{Weight estimators for the first population. The X-axis represents the covariate variable $x$ and the Y-axis depicts the weight function $w_i(x)$. The label "source" represents the true value of the weight function. The green wiggly curve represents the pointwise weight estimator. The purple smooth curve represents the weight estimator based on the sieve extremum estimation.}\label{fig:weight estimation}
\end{figure}
\begin{figure}[H]
    \centering
    \includegraphics[width=0.5\textwidth]{"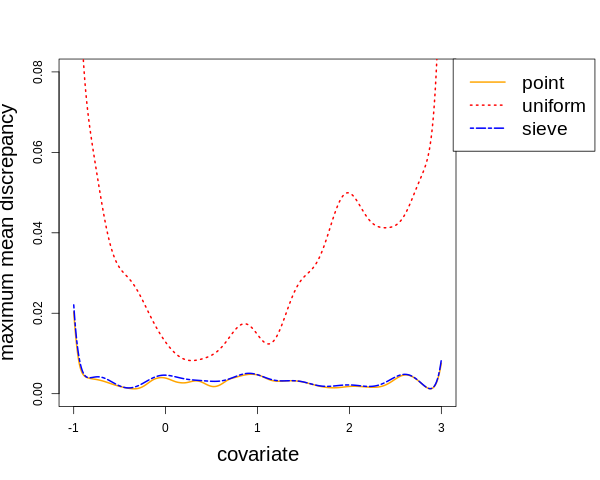"}
    \caption{Estimated CMMD between the weighted source populations and the target population. The X-axis represents the covariate variable $x$ and the Y-axis depicts the CMMD $\hat{d}(x,w(x))$ estimated by Eqn.~\ref{equation: CMMD calculation}. The labels "sieve" and "point" represent the CMMDs based on the sieve weight estimator and the pointwise weight estimator, respectively. The label "uniform" refers to the CMMD based on the uniform weight.}\label{fig: conditional maximum mean discrepancy between weighted source populations and the target population.}
\end{figure}


The coverage rate of the asymptotic confidence interval in Eqn.~\ref{equation: confidence interval} is summarized in Table \ref{table: coverage rate of confidence interval}. It demonstrates that the asymptotic confidence interval can approach the nominal coverage rate when the sample size is sufficiently large.
\begin{table}[!ht]
	\centering
	\caption{The coverage rate of the asymptotic confidence interval.}\label{table: coverage rate of confidence interval}
	\begin{tabular}{lll}
		\hline
		Sample size & 95\%CI & 90\%CI \\ \hline
		4000 & 0.94 & 0.91 \\ 
		3000 & 0.93 & 0.87 \\ 
		2000 & 0.94 & 0.87 \\ \hline
	\end{tabular}
\end{table}

\subsection{Real Data Experiment}

We apply our methods to the microcredit data from \citet{meager2019understanding}. It contains the data from seven RCTs conducted in seven countries in order to evaluate the impact of the expansion of access to microcredit on household business profit, expenditures and revenues. The goal of the RCTs is to investigate whether the expansion of access to microcredit encourages the poor household to open up new business. 

The RCTs are conducted in seven countries: Mexico, Mongolia, Bosnia, India, Morocco, Philippines and Ethiopia. Different RCTs contain different outcome variables and covariate variables. In order to let the source populations and the target population share the largest number of outcome variables, we choose the data from five countries: Mexico, Mongolia, Bosnia, India and Morocco. In each round of the experiment we choose one country as the target population and set the other countries as the source populations. We set aside the treatment group data from the target population as the test data. The estimated treatment group expectation of the target population is treated as the ground truth value. 

There are five outcome variables shared by all of the RCTs: profit, consumption, expenditures, temptation and revenues. In addition, there are three covariate variables (i.e., existing business, income, assets) which are shared by all the RCTs. Because the economic outcome variable may take zero or negative values, we use the inverse hyperbolic sine transformation to transform the outcome variable \citep{Pence2006TheRO}. Following the recommendation of \citep{Egami2020ElementsOE}, in order to compare the results between different randomized controlled trials, we use the mean and standard deviation
of the control group potential outcome to standardize the outcome variables. When we choose one variable as the outcome variable, we let the other variables be the covariate variables. Because the covariate is 7-th dimensional, it may be difficult to choose the appropriate sieve basis functions. Consequently, we only apply the pointwise weight estimation method. 

For each outcome variable, we calculate the mean squared error of the experimental results. The mean squared error of the real data experiment is summarized in Table \ref{table: Mean squared error of Meager dataset}. The column names represent the outcome variable of each experiment. As shown in Table \ref{table: Mean squared error of Meager dataset}, our method based on the constrained weights outperforms the baseline methods in most situations.  

\begin{table}[H]
	\centering
	\caption{Mean squared error of the real data experiment.}\label{table: Mean squared error of Meager dataset}		\begin{tabular}{ l l l l l l }
		\hline
		Outcome & Profit & Consumption & Expenditures & Temptation & Revenues \\ \hline
		Constrained & 0.018(0.021) & 0.01(0.014) & 0.024(0.042) & 0.009(0.009) & 0.033(0.036)  \\ 
		Unconstrained & 0.097(0.13) & 0.019(0.028) & 0.397(0.455) & 0.083(0.178) & 0.366(0.465)  \\ 
		Uniform & 0.02(0.019) & 0.008(0.011) & 0.033(0.041) & 0.01(0.009) & 0.083(0.086)  \\ 
		Pool & 0.098(0.14) & 0.025(0.04) & 0.033(0.007) & 0.014(0.017) & 0.986(1.228) \\ \hline
	\end{tabular}
\end{table}


\section{Conclusion}
In this paper we have discussed the main assumptions of previous works and proposed two alternative identification assumptions when the mean exchangeability assumption is violated. By constructing a synthetic treatment group for the target population, we have introduced a complementary approach to achieve transportability. We have estimated the weights by minimizing the CMMD through a quadratic programming algorithm. The sieve extremum estimation can alleviate the problem of overfitting caused by the pointwise weight estimation. Formulating the problem as the sieve semiparametric two-step estimation, we have established the asymptotic normality of the proposed estimator. Experiments on both the simulated and the real data have demonstrated the effectiveness of our method compared with the baseline methods.

\setcitestyle{numbers}

\bibliography{samples.bib}




\newpage

\appendix

\section{Proofs of the Results in Sections \ref{section: Identification without the mean exchangeability assumption} and \ref{section: Weight estimation by minimizing the CMMD}}

\subsection{Proof of Proposition \ref{proposition: weaker than the distribution exchangeability assumption}}

\begin{proof} 
    Due to the distribution exchangeability assumption, we have that
    \begin{align*}
    \Pr(Y(0) \leq y|X=x,D=0) = \Pr(Y(0) \leq y|X=x,D=i)    
    \end{align*}
    for all $y$ and $i$. 
     It implies that 
    \begin{align*}
 \Pr(Y(0) \leq y|X=x,D=0)&=\left(\sum_{i=1}^{N}w_i^*(x)\right) \Pr(Y(0) \leq y|X=x,D=0)\\
 &=\sum_{i=1}^{N}w_i^*(x) \Pr(Y(0) \leq y|X=x,D=i).    
    \end{align*}
    as long as the weights sum up to one such that $\sum_{i=1}^{N}w_i^*(x)=1$.
\end{proof}

\subsection{Proof of Theorem \ref{proposition: identification}}
\begin{proof} 
	Due to the law of total expectation, we have
	\begin{align*}
		&\mathbb{E}[Y(1)|X=x,D=0]\\
            &=\int \mathbb{E}[Y(1)|Y(0)=y,X=x,D=0]\,\mathrm{d}\Pr(Y(0)\leq y|X=x,D=0)\\
		&=\int \mathbb{E}[Y(1)|Y(0)=y,X=x,D=0]\, \mathrm{d}\left(\sum_{i=1}^{N}w^*_i(x)\Pr(Y(0)\leq y|X=x,D=i)\right)\\
		&=\sum_{i=1}^{N}w^*_i(x)\int \mathbb{E}[Y(1)|Y(0)=y,X=x,D=i]\,\mathrm{d}\Pr(Y(0)\leq y|X=x,D=i)\\
		&=\sum_{i=1}^{N}w^*_i(x)\mathbb{E}[Y(1)|X=x,D=i],
	\end{align*}
	where the first and the fourth equality follow from the law of total expectation, the second equality follows from Assumption \ref{asumption:mixture distribution} and the third equality follows from Assumption \ref{assumption: the invariance of counterfactual expectation transition mechanism}.
	
	Due to the strong ignorability assumption and the SUTVA assumption, we have
	\begin{align*}
		\mathbb{E}[Y(1)|X=x,D=i]=\mathbb{E}[Y(1)|X=x,A=1,D=i]=\mathbb{E}[Y|X=x,A=1,D=i],
	\end{align*}
	where $i\in\{1, 2, \ldots, N\}$. The result of Theorem \ref{proposition: identification} follows naturally.	
\end{proof}
\subsection{Proof of Proposition \ref{proposition: generalize the identification formula under the mean exchangeability assumption}}
\begin{proof} 
    When Assumptions \ref{assumption: the invariance of counterfactual expectation transition mechanism} and \ref{asumption:mixture distribution} are violated but the mean exchangeability assumption holds, we have
	\begin{align*}
		\mathbb{E}[Y(1)|X=x,D=0]=\mathbb{E}[Y(1)|X=x,D=i],   
	\end{align*}
	for $i\in\{1, 2, \ldots, N\}$.
	
	Due to the fact that
	\begin{align*}
		\sum_{i=1}^{N}w_i^*(x)\mathbb{E}(Y(1)|X=x,D=i)&= \left(\sum_{i=1}^{N}w_i^*(x)\right)\mathbb{E}(Y(1)|X=x,D=0)\\
		&= \mathbb{E}(Y(1)|X=x,D=0),  
	\end{align*}
	the identification formulae (Eqns.~\ref{equation: 1} and \ref{equation: 2}) still hold as long as $\sum_{i=1}^{N}w^*_i(x)=1$.
\end{proof}

\subsection{Proof of Example \ref{example: violation of mean exchangeability}}\label{subsection: proof of example 2}
\begin{proof} 
	Due to the fact that:
	
	{\centering
		$ \displaystyle
		\begin{aligned} 
			&\mathbb{E}[Y(1)|X=x,D=i]\\
			&=\int \mathbb{E}[Y(1)|Y(0)=y,X=x,D=i]P(Y(0)=y|X=x,D=i) dy,		
		\end{aligned}$ 
		\par}
	we have
	\begin{align*}
		\mathbb{E}[Y(1)|X=x,D=i]&=E\{\mathbb{E}[Y(1)|X,Y(0),D]|X=x,D=i\}\\
		&=c(x)a_ix+c(x)b_i+d(x).
	\end{align*}
	Through some algebra, we can compare the potential outcome means between the source populations as follows:
	\begin{align*}
		\mathbb{E}[Y(1)|X=x,D=i]-\mathbb{E}[Y(1)|X=x,D=j]&=c(x)\left[(a_i-a_j)x+b_i-b_j\right]\\
		\mathbb{E}[Y(0)|X=x,D=i]-\mathbb{E}[Y(0)|X=x,D=j]&=(a_i-a_j)x+b_i-b_j
	\end{align*} 
	We can also compare the conditional average treatment effect between the $i$-th trial and the $j$-th trial:
	\begin{align}
		\mathbb{E}[Y(1)-Y(0)|X=x,D=i]-\mathbb{E}[Y(1)-Y(0)|X=x,D=j]=(c(x)-1)\left[(a_i-a_j)x+b_i-b_j\right].     \end{align}
	Consequently, if $\left\{\begin{aligned}
		c(x)&\neq 0\\
		c(x)&\neq 1\\
		a_i&\neq a_j
	\end{aligned}\right.$, then $\left\{\begin{aligned}
		\mathbb{E}[Y(1)-Y(0)|X=x,D=i]&\neq \mathbb{E}[Y(1)-Y(0)|X=x,D=j]\\
		\mathbb{E}[Y(0)|X=x,D=i]&\neq \mathbb{E}[Y(0)|X=x,D=j]\\
		\mathbb{E}[Y(1)|X=x,D=i]&\neq \mathbb{E}[Y(1)|X=x,D=j]
	\end{aligned}\right.$
\end{proof}

\subsection{Proof of Theorem \ref{proposition: estimation}}

\begin{proof} 
	As mentioned in the main text, according to the linearity of integral operator, we have
	\begin{align*}
		\mu_{\sum_{i=1}^{N}w_i(x)P_{Y(0)|X=x,D=i}}&=\int \ell(y,\cdot)[\sum_{i=1}^{N}w_i(x)\Pr(Y(0)=y|X=x,D=i)]dy\\
		&=\sum_{i=1}^{N}w_i(x)\int \ell(y,\cdot)\Pr(Y(0)=y|X=x,D=i)dy\\
		&=\sum_{i=1}^{N}w_i(x)\mu_{Y(0)|X=x,D=i}.
	\end{align*}
	
	Substituting the CME estimator into the definition of the CMMD, we obtain that
	\begin{align*}
		\lefteqn{\hat{d}(x,w(x))} \\
		& =  \Vert \hat{\mu}_{\sum_{i=1}^{N}w_i(x)P_{Y(0)|X=x,D=i}}-\hat{\mu}_{P_{Y(0)|X=x,D=0}}\Vert^2_{\mathcal{G}}\\
		& =  \langle \sum_{i=1}^{N}w_i(x)\hat{\mu}_{Y(0)|X=x,D=i}-\hat{\mu}_{Y(0)|X=x,D=0},\sum_{i=1}^{N}w_i(x)\hat{\mu}_{Y(0)|X=x,D=i}-\hat{\mu}_{Y(0)|X=x,D=0} \rangle_\mathcal{G}\\
		& =  \langle C(x)w(x)- l_0^TM_0k_0(x)
		, C(x)w(x)- l_0^TM_0k_0(x)\rangle_\mathcal{G}\\	
		& = \left(C(x)w(x)- l_0^TM_0k_0(x)
		\right)^T\left(C(x)w(x)- l_0^TM_0k_0(x)\right)\\
		&=w(x)^TC(x)^TC(x)w(x)-2w(x)^TC(x)^T\begin{bmatrix}
			l_0^TM_0k_0(x)
		\end{bmatrix}  + k_0^T(x)M_0l_0l_0^TM_0k_0(x)\\
		&=w(x)^T\hat{A}(x)w(x)-2w(x)^T\hat{b}(x)+k_0^T(x)M_0l_0l_0^TM_0k_0(x),
	\end{align*}
 where we refer to $C(x)$ as $\left[l_1^TM_1k_1(x),\cdots, l_N^TM_Nk_N(x)\right]$.
 
 We can verify that $\hat{A}(x)=\begin{bmatrix}
		k_1^T(x)M_1l_1 \\
		\vdots \\
		k_N^T(x)M_Nl_N
	\end{bmatrix}\left[l_1^TM_1k_1(x),\cdots, l_N^TM_Nk_N(x)\right]$ and\\ $\hat{b}(x)=\begin{bmatrix}
		k_1^T(x)M_1l_1 \\
		\vdots \\
		k_N^T(x)M_Nl_N
	\end{bmatrix}\begin{bmatrix}
		l_0^TM_0k_0(x)
	\end{bmatrix}$.
	
	Recall that $\ell_i^TM_ik_i(x)$ is an estimator for $\mu_{Y(0)|X=x,D=i}$, for $i=0, 1, \ldots, N$. Consequently, $[\hat{A}(x)]_{i,j}=\langle\hat{\mu}_{Y(0)|X=x,D=i},\hat{\mu}_{Y(0)|X=x,D=j}\rangle_\mathcal{G}$ is a plug-in estimator for $[A(x)]_{i,j}=\langle\mu_{Y(0)|X=x,D=i},\mu_{Y(0)|X=x,D=j}\rangle_\mathcal{G}.$ And $[\hat{b}(x)]_i=\langle\hat{\mu}_{Y(0)|X=x,D=i},\hat{\mu}_{Y(0)|X=x,D=0}\rangle_\mathcal{G}$ is a plug-in estimator for $[b(x)]_i=\langle\mu_{Y(0)|X=x,D=i},\mu_{Y(0)|X=x,D=0}\rangle_\mathcal{G}.$   	
	
	Because $\hat{A}(x)=\begin{bmatrix}
		k_1^T(x)M_1l_1 \\
		\vdots \\
		k_N^T(x)M_Nl_N
	\end{bmatrix}\begin{bmatrix}
		k_1^T(x)M_1l_1 \\
		\vdots \\
		k_N^T(x)M_Nl_N
	\end{bmatrix}^T$ is a positive semidefinite matrix, $\hat{d}(x,w(x))$ is a convex quadratic function with respect to $w(x)$.
	
	We can solve the following equation:
	\begin{align*}
		\frac{\partial \hat{d}(x,w(x))}{\partial w(x)}= 2\hat{A}(x)w(x)-2\hat{b}(x)=0.   
	\end{align*}
	
	It follows that $\hat{d}(x,w(x))$ is minimized when $w(x)=\hat{A}(x)^{-1}\hat{b}(x)$.

	We can construct a plug-in estimator for $\mathbb{E}[d(X,w(X)|D=0]$ as
	\begin{align*}
		\frac{1}{n_0}\sum_{i=1}^{n_0}\hat{d}(x_{0,i},w(x_{0,i})).    
	\end{align*}
	When we use the sieve method to estimate the weights, we can substitute $w(x)=V(x)^T\beta_w$ into the above equation to obtain Eqn.~\ref{equation: minimizing average cmmd}.
\end{proof}

\section{Derivations of the asymptotic variance and its corresponding estimator based on the parametric belief }
In order to estimate the asymptotic variance $V_{\theta}$, we can investigate the problem from a parametric perspective. \citet{chen2015sieve} and \citet{ackerberg2012practical} establish the numerical equivalence result of the asymptotic variance estimator based on the parametric belief and the corresponding estimator under the nonparametric specification. Please refer to Section 5 of \citet{chen2015sieve} for more details.  

Under the parametric belief, we presume that the true value of the unknown parameters lies in the linear sieve space such that $w^*_i(\cdot)={\beta_{o,w_i}}^TP_i(\cdot)$ and $g^*_i(\cdot)={\beta_{o,g_i}}^TQ_i(\cdot),$ for $i=1,\ldots,N$.

Representing the nuisance parameter by its corresponding sieve coefficients $\beta_w$ and $\beta_g$, we can reformulate the criterion functions by adding the subscript $P$ as 
\begin{align*}
	\varphi_{w,P}(Z,\beta_{w})=\varphi_{w}(Z,V(X)^T\beta_w),\\
	\varphi_{g_i,P}(Z,\beta_{g_i})=\varphi_{g_i}(Z,{\beta_{g_i}}^TQ_i(\cdot)).
\end{align*}

Denote $\beta=[\beta_w^T,\beta_g^T]^T$ and $\beta_{o,P}=[\beta_{o,w}^T,\beta_{o,g}^T]^T$. The criterion function under the parametric belief is 
\begin{align*}
	\varphi_P(Z,\beta)=\varphi_{w,P}(Z,\beta_{w})+\sum_{i=1}^N \varphi_{g_i,P}(Z,\beta_{g_i}).   
\end{align*}

The second step moment function can be reformulated as 
\begin{align*}
	g_P(Z,\theta,\beta_{w},\beta_{g})=g\left(Z,\theta,\{{\beta_{w_i}}^TP_i(\cdot)\}_{i=1}^N,\{{\beta_{g_i}}^TQ_i(\cdot)\}_{i=1}^N\right).    
\end{align*}

The expression of the adjustment term in Eqn.~\ref{equation: adjustment term based on the parametric belief} based on the parametric belief comes from applying the sieve semiparametric two-step estimation method. According to Section 5.1 of \citet{chen2015sieve}, let 
\begin{align*}
	&R_{P}=-\mathbb{E}[\frac{\partial^2 \varphi_{P}(Z,\beta_{o,P})}{\partial \beta\partial \beta^T}]\\
    &\quad=\begin{bmatrix}
		-\mathbb{E}[\frac{\partial^2 \varphi_{w,P}(Z,\beta_{o,w})}{\partial \beta_w\partial \beta_w^T}] &  &  &  \\
		& -\mathbb{E}[\frac{\partial^2 \varphi_{g_1,P}(Z,\beta_{o,g_1})}{\partial \beta_{g_1}\partial \beta_{g_1}^T}] &  &  \\
		&  & \ddots &  \\
		&  &  & -\mathbb{E}[\frac{\partial^2 \varphi_{g_N,P}(Z,\beta_{o,g_N})}{\partial \beta_{g_N}\partial \beta_{g_N}^T}]
	\end{bmatrix},\\
	&\Gamma_{2,P}=\frac{\partial \mathbb{E}[g_P(Z,\theta_0,\beta_{o,P})]}{\partial \beta^T}\\
 &\quad\quad = \left[\frac{\partial \mathbb{E}[g_P(Z,\theta_0,\beta_{o,w},\beta_{o,g})]}{\partial \beta_w^T},
	\frac{\partial \mathbb{E}[g_P(Z,\theta_0,\beta_{o,w},\beta_{o,g})]}{\partial \beta_{g_1}^T},
	\cdots ,
	\frac{\partial \mathbb{E}[g_P(Z,\theta_0,\beta_{o,w},\beta_{o,g})]}{\partial \beta_{g_N}^T}
	\right],\\
	&\frac{\partial \varphi_{P}(Z,\beta_{o,P})}{\partial \beta}=\begin{bmatrix}
		\frac{\partial \varphi_{w,P}(Z,\beta_{o,w})}{\partial \beta_w} \\
		\frac{\partial \varphi_{g_1,P}(Z,\beta_{o,g_1})}{\partial \beta_{g_1}} \\
		\vdots \\
		\frac{\partial \varphi_{g_N,P}(Z,\beta_{o,g_N})}{\partial \beta_{g_N}}
	\end{bmatrix}.
\end{align*}
Substituting the above equation into  Section 5.1 of \citet{chen2015sieve}, the adjustment term based on parametric belief can be decomposed into two parts
\begin{align}\label{equation: asymptotic variance based on parametric belief}
	\Delta_P(Z,\beta_{o,w},\beta_{o,g})[v_P^*]&=\Gamma_{2,P}(R_{P})^{-1}\frac{\partial \varphi_{P}(Z,\beta_{o,P})}{\partial \beta}\notag\\
	&=\Gamma_{2,w}(R_w)^{-1}\frac{\partial \varphi_{w,P}(Z,\beta_{o,w})}{\partial \beta_w}+\sum_{i=1}^{N}\Gamma_{2,g_i}(R_{g_i})^{-1}\frac{\partial \varphi_{g_i,P}(Z,\beta_{o,g_i})}{\partial \beta_{g_i}},    
\end{align}
where $v^*_P$ denotes the Riesz representer under the parametric belief. 

The first part with the subscript $w$ reflects the influence of the weight estimation. We let 
\begin{align*}
	&R_w=-\mathbb{E}[\frac{\partial^2 \varphi_{w,P}(Z,\beta_{o,w})}{\partial \beta_w\partial \beta_w^T}]=2\mathbb{E}[V(X)A(X)V(X)^T|D=0],\\
	&\Gamma_{2,w}=\frac{\partial \mathbb{E}[g(Z,\theta_0,\beta_{o,w},\beta_{o,g})]}{\partial \beta_w^T}=-E\{[g^*_1(X)P_1^T(X),\cdots,g^*_N(X)P_N^T(X)]|D=0\},\\
	&\frac{\partial \varphi_{w,P}(Z,\beta_{o,w})}{\partial \beta_w}=\frac{I_{D=0}}{P(D=0)}\left\{-2V(X)A(X)V(X)^T\beta_{o,w}+2V(X)b(X)\right\}.
\end{align*}
Recall that $V(x)=\mathrm{diag}(P_1(x),\cdots,P_N(x))$ and $P_i(x)$ is the basis function for estimating the weight $w_i^*(x)$ of the $i$-th source population.

The second part with the subscript $g$ reflects the influence of the outcome regression estimation. We let
\begin{align*}
	&R_{g_i}=-\mathbb{E}[\frac{\partial^2 \varphi_{g_i,P}(Z,\beta_{o,g_i})}{\partial \beta_{g_i}\partial \beta_{g_i}^T}]=2\mathbb{E}[Q_i(X)Q_i(X)^T|A=1,D=i],\\
	&\Gamma_{2,g_i}=\frac{\partial \mathbb{E}[g(Z,\theta_0,\beta_{o,w},\beta_{o,g})]}{\partial \beta_{g_i}^T}=-\mathbb{E}[w^*_i(X)Q_i(X)^T|D=0],\\
	&\frac{\partial \varphi_{g_i,P}(Z,\beta_{o,g_i})}{\partial \beta_{g_i}}=\frac{I_{A=1,D=i}}{P(A=1,D=i)}\{-2Q_i(X)Q_i(X)^T\beta_{o,g_i}+2Q_i(X)Y\}.
\end{align*}  
Recall that $Q_i(x)$ is the basis function for estimating the outcome regression function $g_i^*(x)$.

We can construct the plug-in estimator to estimate the adjustment term in Eqn.~\ref{equation: asymptotic variance based on parametric belief} such that
\begin{align*}
	\hat{\Delta}_P(Z,\hat{\beta}_{w},\hat{\beta}_{g})[\hat{v}_P]=\hat{\Gamma}_{2,w}(\hat{R}_w)^{-1}\frac{\partial \hat{\varphi}_{w,P}(Z,\hat{\beta}_w)}{\partial \beta_w}+\sum_{i=1}^{N}\hat{\Gamma}_{2,g_i}(\hat{R}_{g_i})^{-1}\frac{\partial \varphi_{g_i,P}(Z,\hat{\beta}_{g_i})}{\partial \beta_{g_i}}.   
\end{align*}

\section{Proofs of the Results in Section \ref{section: Sieve semiparametric two-step estimation}}
\subsection{Proof of Theorem \ref{theorem: asymptotic normality}}\label{subsection: proof of theorem 5}

Denote $\mathcal{S}$ as the infinite dimensional parameter space of the nuisance parameter $h$ with a problem specific pseudo-metric $\|\cdot\|_{\mathcal{S}}$. According to \citet{chen2015sieve}, we define a local pseudo metric 
\begin{align*}
	\|h-h_0\|=\{-\frac{\partial^2 Q(h_0+\tau(h-h_0))}{\partial \tau^2}|_{\tau=0}\}^{1/2}
\end{align*}
for any $h\in \mathcal{S}$. Let $\mathcal{V}$ be the closed linear span of $\mathcal{S}-h_0$ under $\|\cdot\|$.

Note that $v^*\in\mathcal{V}$ in Theorem \ref{theorem: asymptotic normality} is the Riesz representer of the Gateaux derivative $\Gamma_2(\theta_0,h_0)[v]=\frac{\partial \mathbb{E}[g(Z,\theta_0,h_0+\tau v)]}{\partial \tau}|_{\tau=0}$ with respect to this pseudo-metric such that 
\begin{align*}
	\Gamma_2(\theta_0,h_0)[v]=\frac{\partial \mathbb{E}[g(Z,\theta_0,h_0+\tau v)]}{\partial \tau}|_{\tau=0}=\langle v^*,v\rangle 
\end{align*}
for any  $v\in \mathcal{V}$, where $\langle\cdot,\cdot\rangle$ is the inner product induced by the norm $\|\cdot\|$. The definition of the Riesz representer  $v^*$ follows from Equation (16) of \citep{chen2015sieve}.

Denote $v_{n_T}^*\in \mathcal{V}_{n_T}$ as the sieve Riesz representer with respect to the linear sieve space $\mathcal{V}_{n_T}$ such that 
\begin{align*}
	\Gamma_2(\theta_0,h_0)[v]=\frac{\partial \mathbb{E}[g(Z,\theta_0,h_0+\tau v)]}{\partial \tau}|_{\tau=0}=\langle v_{n_T}^*,v\rangle 
\end{align*}
for any  $v\in \mathcal{V}_{n_T}$. The linear sieve space $\mathcal{V}_{n_T}=\mathcal{S}_{n_T}$ becomes dense in $\mathcal{V}$ as $n_T\to \infty$. The definition of the sieve Riesz representer  $v_{n_T}^*$ follows from Equation (17) of \citep{chen2015sieve}.

In order for Theorem \ref{theorem: asymptotic normality} to hold, we need the following regularity assumptions from \citet{chen2015sieve}.

\begin{assumptionA}[Assumption A.1 of \citet{chen2015sieve}]\label{Assumption A.1}
	
	The first-step sieve extremum estimator $\widehat{h}$ satisfies:
	\begin{enumerate}
		\item $\left|n_T^{-1} \sum_{i=1}^{n_T} \Delta\left(Z_i, h_o\right)\left[v_{n_T}^*\right]-\left\langle v_{n_T}^*, \widehat{h}-h_o\right\rangle\right|=o_p\left(n_T^{-1 / 2}\right)$;
		\item $\left\|\widehat{h}-h_o\right\| \times \left\|v_{n_T}^*-v^*\right\|=o_p\left(n_T^{-1 / 2}\right)$. 
	\end{enumerate}
\end{assumptionA}
Assumption \ref{Assumption A.1} connects the adjustment term $\Delta(z,h)$ with the sieve Riesz representer $v^*_{n_T}$. It is indispensable for the validity of the asymptotic confidence interval. \citet{chen2015sievewald} adopted a similar assumption (Assumption 3.6) about the local quadratic approximation to the sample criterion difference. According to \citet{chen2015sieve}, Assumption \ref{Assumption A.1} is satisfied by both sieve M estimation and sieve minimum distance estimation under low level conditions.
\begin{assumptionA}[Assumption A.2 of \citet{chen2015sieve}]\label{assumption:A.2}
	Suppose that $\theta_o$ satisfies $G\left(\theta_o, h_o\right)=\mathbb{E}[g(Z,\theta_0,h_0)]=0$, that $\widehat{\theta}-\theta_o=o_p(1)$, and that (i) $\Gamma_1\left(\theta, h_o\right)=\frac{\partial \mathbb{E}[g(Z,\theta,h_0)]}{\partial \theta}$ exists in a neighborhood of $\theta_o$ and is continuous at $\theta_o, \Gamma_1^{\prime} W \Gamma_1$ is nonsingular, where $W$ is the weighting matrix of the second step GMM estimation; (ii) the pathwise derivative $\Gamma_2\left(\theta, h_o\right)\left[h-h_o\right]=\frac{\partial \mathbb{E}[g(Z,\theta,h_0+\tau (h-h_0))]}{\partial \tau}|_{\tau=0}$ exists in all directions $\left[h-h_o\right]$ and satisfies
	$$
	\left|\Gamma_2\left(\theta, h_o\right)\left[h-h_o\right]-\Gamma_2\left(\theta_o, h_o\right)\left[h-h_o\right]\right| \leq\left|\theta-\theta_o\right|\times o(1)
	$$
	for all $\theta$ with $\left|\theta-\theta_o\right|=o(1)$ and all $h$ with $\left\|h-h_o\right\|_{\mathcal{S}}=o(1)$; either (iii)
	$$
	\left|G\left(\theta, \widehat{h}\right)-G\left(\theta, h_o\right)-\Gamma_2\left(\theta, h_o\right)\left[\widehat{h}-h_o\right]\right|=o_p\left({n_T}^{-\frac{1}{2}}\right)
	$$
	for all $\theta$ with $\left|\theta-\theta_o\right|=o(1)$; or (iii)' there are some constants $c \geq 0, \epsilon_1>0, \epsilon_2>1$ such that
	$$
	\left|G(\theta, h)-G\left(\theta, h_o\right)-\Gamma_2\left(\theta, h_o\right)\left[h-h_o\right]\right| \leq c\left\|h-h_o\right\|_{\mathcal{S}}^{\epsilon_1}\left\|h-h_o\right\|^{\epsilon_2}
	$$
	for all $\theta$ with $\left|\theta-\theta_o\right|=o(1)$, all $h$ with $\left\|h-h_o\right\|_{\mathcal{S}}=o(1), c\left\|\widehat{h}-h_o\right\|_{\mathcal{S}}^{\epsilon_1}\left\| \widehat{h}-h_o\right\|^{\epsilon_2}=o_p\left({n_T}^{-\frac{1}{2}}\right)$;
	(iv) Denote $G_{n_T}(\theta,h)=\frac{1}{n_T}\sum_{i=1}^{n_T} g(Z_i,\theta,h)$. For all sequences of positive numbers $\left\{\kappa_{n_T}\right\}$ with $\kappa_{n_T}=o(1)$
	$$
	\sup _{\left|\theta-\theta_o\right|<\kappa_{n_T},\left\|h-h_o\right\|_{\mathcal{S}}<\kappa_{n_T}} \frac{\left|G_{n_T}(\theta, h)-G(\theta, h)-G_{n_T}\left(\theta_o, h_o\right)\right|}{{n_T}^{-1 / 2}+\left|G_{n_T}(\theta, h)\right|+|G(\theta, h)|}=o_p(1)
	$$
	(v) ${n_T}^{-\frac{1}{2}} \sum_{i=1}^{n_T}\left\{g\left(Z_i, \theta_o, h_o\right)+\Delta\left(Z_i, h_o\right)\left[v_{n_T}^*\right]\right\} \rightarrow_d \mathcal{N}\left(0, V_1\right)$.   
\end{assumptionA}
Assumption \ref{assumption:A.2} is about the regularity conditions of the pathwise derivative and the adjustment term. It controls the magnitude of the remainder term of the Taylor expansion which reflects the smoothness of the corresponding functional. It is also indispensable for the validity of the asymptotic normality result.
\begin{proof}[Proof of Theorem \ref{theorem: asymptotic normality}]
	Theorem \ref{theorem: asymptotic normality} is an application of Theorem 2.1 of \citet{chen2015sieve} to the problem of treatment group expectation estimation.
	
	We first verify that the first step nonparametric estimation of weights and outcome regression functions belongs to the class of Sieve M estimation. 
	
	Due to the fact that
	\begin{align}\label{equation: the expression of Q(h)}
		Q(h)&=-E\left[d(X,w(X))|D=0\right]-\sum_{i=1}^{N}\mathbb{E}[(Y-g_i(X))^2|A=1,D=i]\notag\\
		&=-E\left[\frac{I_{D=0}}{P(D=0)}d(X,w(X))\right]-\sum_{i=1}^{N}E\left[\frac{I_{A=1,D=i}}{P(A=1,D=i)}(Y-g_i(X))^2\right]\\
		&=E\left[\varphi_w(Z,w(X))+\sum_{i=1}^{N}\varphi_{g_i}(Z,g_i)\right],\notag
	\end{align}
	$\hat{Q}(h)$ is the sample analog of the criterion function of the sieve M estimation problem 
	\begin{align*}
		\underset{h}{\arg \sup}~E\left[\varphi_w(Z,w(X))+\sum_{i=1}^{N}\varphi_{g_i}(Z,g_i)\right].    
	\end{align*}
	
	Eqn.~\ref{equation: the expression of Q(h)} follows from the fact that
	\begin{align*}
		&\mathbb{E}\left[\frac{I_{D=0}}{P(D=0)}d(X,w(X))\right]\\
		&=P(D=0)\mathbb{E}\left[\frac{I_{D=0}}{P(D=0)}d(X,w(X))|D=0\right]+P(D\neq 0)\mathbb{E}\left[\frac{I_{D=0}}{P(D=0)}d(X,w(X))|D\neq 0\right]\\
		&=\mathbb{E}\left[d(X,w(X))|D=0\right]
	\end{align*}
	and \begin{align*}
		&\mathbb{E}\left[\frac{I_{A=1,D=i}}{P(A=1,D=i)}(Y-g_i(X))^2\right]\\
		&=P(A=1,D=i)\mathbb{E}\left[\frac{I_{A=1,D=i}}{P(A=1,D=i)}(Y-g_i(X))^2|A=1,D=i\right]\\
		&\quad +P(I_{A=1,D=i}=0)\mathbb{E}\left[\frac{I_{A=1,D=i}}{P(A=1,D=i)}(Y-g_i(X))^2|I_{A=1,D=i}=0\right]\\
		&=\mathbb{E}[(Y-g_i(X))^2|A=1,D=i].
	\end{align*}
	
	Since the linear sieve space $\mathcal{S}_{n_T}$ of weights and outcome regression functions are bounded, the sample analog of the criterion function $\hat{Q}(h)$ satisfies \begin{align*}
		\sup_{h\in \mathcal{S}_{n_T}}~\|\hat{Q}(h)-Q(h)\|=o_p(1).
	\end{align*}
	By Equation (21) of \citet{chen2015sieve}, the adjustment term of the first step sieve M estimation is defined as 
	\begin{align*}
		\Delta(Z,h)[v]=\frac{\partial \varphi(Z,h+\tau v)}{\partial \tau}|_{\tau=0}.
	\end{align*}

	According to \citet{chen2015sieve}, the first step sieve M estimator $\hat{h}$ satisfies Assumptions \ref{Assumption A.1}.  
	
	When we use the sample average estimator $\frac{\sum_{i=1}^{n_T}I_{D_i=0}}{n_T}$ for $\Pr(D=0)$, the second step GMM estimator of $\theta_0$ is equivalent to the synthetic treatment group estimator in Eqn.~\ref{equation: treatment group expectation estimator} since
	\begin{align*}
		&\hat{\theta}=\frac{1}{n_0}\sum_{j=1}^{n_0}\sum_{i=1}^{N}\hat{w}_i(x_{0,j})\hat{g}_i(x_{0,j}),\\
		&=\underset{\theta\in \mathcal{R}}{\arg \inf}~\left(\frac{1}{n_T}\sum_{i=1}^{n_T}g(Z_i,\theta,\hat{h})\right)'W\left(\frac{1}{n_T}\sum_{i=1}^{n_T}g(Z_i,\theta,\hat{h})\right)\\
		&=\underset{\theta\in \mathcal{R}}{\arg \inf}~\left[\frac{1}{n_T}\sum_{j=1}^{n_T}\left(\theta-E(X_j)\right)\right]'W\left[\frac{1}{n_T}\sum_{j=1}^{n_T}\left(\theta-E(X_j)\right)\right]
	\end{align*}
	for any weighting matrix $W$, where we let $E(x) = \frac{I_{D_j=0}}{\frac{\sum_{i=1}^NI_{D_i=0}}{n_T}}\sum_{i=1}^{N}w_i(x)g_i(x)$.
	
	Because $g(Z,\theta,h)=\theta-\frac{I_{D=0}}{P(D=0)}\sum_{i=1}^{N}w_i(X)g_i(X)$ is a linear function with respect to $\theta$ and a smooth function with respect to $h$, Assumption \ref{assumption:A.2} is naturally satisfied.

	Consequently, Theorem 3 follows from Theorem 3.1 of \citet{chen2015sieve}. 
\end{proof}

\subsection{Proof of Proposition \ref{proposition: asymptotic variance estimator}}\label{subsection: proof of proposition 6}

In order to deal with the weakly dependent data, \citet{chen2015sieve} use a kernel based estimator of $V_1$ as
$$
\widehat{V}_{1, n_T}=\sum_{i=-n_T+1}^{n_T-1} \mathcal{K}\left(\frac{i}{M_{n_T}}\right) \Upsilon_{n_T, i}\left(\widehat{\alpha}_{n_T}\right)\left[\widehat{v}_{n_T}^*, \widehat{v}_{n_T}^*\right],
$$
where $M_{n_T} \rightarrow \infty$ as $n_T \rightarrow \infty$, and
$$
\Upsilon_{n_T, i}\left(\widehat{\alpha}_{n_T}\right)\left[\widehat{v}_{n_T}^*, \widehat{v}_{n_T}^*\right]=\left\{\begin{array}{cl}
	\frac{1}{n_T} \sum_{l=i+1}^{n_T} \widehat{S}_{l, n_T}^* \widehat{S}_{l-i, n_T}^{* \prime} & \text { for } i \geq 0 \\
	\frac{1}{n_T} \sum_{l=-i+1}^{n_T} \widehat{S}_{l, n_T}^* \widehat{S}_{l+i, n_T}^{* \prime} & \text { for } i<0
\end{array}\right.
$$
with $\widehat{S}_{l, n_T}^*=\widehat{S}_l\left(\widehat{\alpha}_{n_T}\right)\left[\widehat{\mathbf{v}}_{n_T}^*\right]$ given in Equation (32) of \citet{chen2015sieve}.

In order for Proposition \ref{proposition: asymptotic variance estimator} to hold, we need the following regularity assumptions from \citet{chen2015sieve}.

\begin{assumptionA}[Assumption B.1 of \citet{chen2015sieve}]\label{Assumption B.1}
	The kernel function $\mathcal{K}(\cdot)$ is symmetric, continuous at zero, and satisfies $\mathcal{K}(0)=$ $1, \sup _x|\mathcal{K}(x)| \leq 1, \int_{\mathbb{R}}|\mathcal{K}(x)| d x<\infty$ and $\int_{\mathbb{R}}|\mathcal{K}(x)||x| d x<\infty$.    
\end{assumptionA}
Denote $\mathcal{N}_{n_T}$ as a local shrinking neighborhood of $h_o:$
\begin{align*}
	\left\{h \in \mathcal{S}_{n_T}\right.\colon\left.\left\|h-h_o\right\|_{\mathcal{S}} \leq \delta_{s, n_T},\left\|h-h_o\right\| \leq \delta_{n_T}\right\},    
\end{align*}
where $\delta_{s, n_T}=o(1)$ and $\delta_{n_T}=o\left(n_T^{-\frac{1}{4}}\right)$.

Denote
\begin{align*}
	S_i(\alpha)[v]=g\left(Z_i, \alpha\right)+\Delta\left(Z_i, h\right)[v] \quad \text { and } \quad \widehat{S}_i(\alpha)[v]=g\left(Z_i, \alpha\right)+\widehat{\Delta}\left(Z_i, h\right)[v]. 
\end{align*}

Let $\;\mathcal{U}_{n_T}=\left\{v \in \mathcal{V}_{n_T}:\|v\|=1\right\}$ and $\delta_{w, n_T}=o(1)$ be a positive sequence such that
\begin{enumerate}
	\item $\sup _{v_1, v_2 \in \mathcal{U}_{n_T}}\left|\left\langle v_1, v_2\right\rangle_{n_T}-\left\langle v_1, v_2\right\rangle\right|=O_p\left(\delta_{w, n}\right)$, where $\left\langle \cdot, \cdot\right\rangle_{n_T}$ is the empirical inner product induced by the empirical semi-norm in Equation (36) of \citet{chen2015sieve};
	\item $\sup _{\alpha \in \mathcal{N}_{n_T}, v \in \mathcal{U}_{n_T}}\left|\Gamma_{2,n}(\theta, h)[v]-\Gamma_{2}\left(\theta_o, h_o\right)[v]\right|=O_p\left(\delta_{w, n_T}\right)$, where the empirical version of the Gateaux derivative $\Gamma_{2, n}$ follows from Equation (35) of \citet{chen2015sieve};
	\item $\lim _{n_T \rightarrow \infty}\left\|v_{n_T}^*\right\|>0$ and $\lim _{n_T \rightarrow \infty} \left\|v_{n_T}^*\right\|<\infty$.
\end{enumerate}

According to \citet{chen2015sieve}, any sieve extremum estimation in the first step satisfies the above regularity conditions.

\begin{assumptionA}[Assumption B.2 of \citet{chen2015sieve}]\label{assumption:B.2}
	Let $\alpha=(\theta,h)$ and $\alpha=(\theta_0,h_0)$. (i) $\left\{Z_i\right\}$ is a strictly stationary strong mixing process with mixing coefficient $\alpha_i$ satisfying $\sum_{i=0}^{\infty} \alpha_i^{2(1 / r-1 / p)}<\infty$ for some $r \in(2,4]$ and some $p>r$; (ii) $\sup _{n_T}\left|S_i\left(\alpha_o\right)\left[v^*_{n_T}\right]\right|<\infty$ and $\sup _{n_T} \mathbb{E}\left[\sup _{v \in \mathcal{U}_{n_T}}\left|\Delta\left(Z, h_o\right)[v]\right|^2\right]<\infty$; (iii) There is a positive sequence $\delta_{n_T}^*=o(1)$ such that
	$$
	E\left[\sup _{\alpha \in \mathcal{N}_{n_T}, v \in \mathcal{U}_{n_T}}\left|S_i(\alpha)[v]-S_i\left(\alpha_o\right)[v]\right|^2\right]=O\left(\delta_{n_T}^* \delta_{w, n_T}\right)$$
	(iv) $\sup _{h \in \mathcal{N}_{n_T}, v \in \mathcal{U}_{n_T}} \frac{1}{n_T} \sum_{i=1}^{n_T}\left|\widehat{\Delta}\left(Z_i, h\right)[v]-\Delta\left(Z_i, h\right)[v]\right|^2=O_p\left(\delta_{n_T}^* \delta_{w, n_T}\right)$;\\
	(v) $M_{n_T} \times \max \left(\delta_{w, n_T}, \delta_{n_T}^*\right)=o(1)$ and $n^{-1 / 2+1 / r} M_{n_T}=o(1)$.
\end{assumptionA}
According to \citet{chen2015sieve}, when the data are independent and identically distributed, we only need $\delta^*_{n_T}\delta_{w,n_T}=o(1)$ for Assumption \ref{assumption:B.2}(iii)(iv)(v) to hold.
\begin{proof}[Proof of Proposition \ref{proposition: asymptotic variance estimator}]
	Proposition \ref{proposition: asymptotic variance estimator} is an application of Theorem 3.1 and Theorem 5.2 of \citet{chen2015sieve} to the treatment group expectation estimation. 
	
	Because $\{Z_i\}_{i=1}^{n_T}$ are independent and identically distributed, we choose the kernel function $\mathcal{K}(x)=1_{x\in[-1,1]}$ which leads to the asymptotic variance estimator $\hat{V}_{\theta}=\frac{1}{n_T}\sum_{i=1}^{n_T}(\hat{S}_{i,P}-\frac{1}{n_T}\sum_{j=1}^{n_T}\hat{S}_{j,P})^2.$
	
	According to Lemma D.1 of \citet{chen2015sieve}, we can establish the numerical equivalence of the adjustment term estimator based on parametric belief and the corresponding estimator under the nonparametric specification.
	\begin{align*}
		\widehat{\Delta}\left(Z, \widehat{h}\right)\left[\widehat{v}_{n_T}^*\right]&=\hat{\Gamma}_{2,w}(\hat{R}_w)^{-1}\frac{\partial \hat{\varphi}_{w,P}(Z,\hat{\beta}_w)}{\partial \beta_w}+\sum_{i=1}^{N}\hat{\Gamma}_{2,g_i}(\hat{R}_{g_i})^{-1}\frac{\partial \varphi_{g_i,P}(Z,\hat{\beta}_{g_i})}{\partial \beta_{g_i}}\\
		&=\hat{\Delta}_P(Z,\hat{\beta}_{w},\hat{\beta}_{g})[\hat{v}_P].
	\end{align*}
	
	Consequently, the asymptotic variance estimator based on parametric belief is numerically equivalent to the corresponding estimator under the nonparametric specification according to Theorem 5.2 of \citet{chen2015sieve}.
	
	When $\{Z_i\}_{i=1}^{n_T}$ are independent and identically distributed, we only need $\delta^*_{n_T}\delta_{w,n_T}=o(1)$ for Assumption \ref{assumption:B.2}(iii)(iv)(v) to hold. In other words, we only need to demonstrate the consistency of the estimator without concerning its convergence rate. Since the sieve score $S_i(\alpha)[v]$ is a smooth function with respect to $\alpha$ and $v$, Assumption \ref{assumption:B.2} (iii) is satisfied.
	
	\citet{li2021improving} obtain the optimal $O(\frac{1}{n})$ learning rates up to a logarithmic factor of conditional mean embedding estimators under certain regularity conditions. 
	By the continuous mapping theorem, we have
	\begin{align*}
		[\hat{A}(x)]_{i,j}&=\langle\hat{\mu}_{Y(0)|X=x,D=i},\hat{\mu}_{Y(0)|X=x,D=j}\rangle_\mathcal{G}\overset{p}{\to}[A(x)]_{i,j}=\langle\mu_{Y(0)|X=x,D=i},\mu_{Y(0)|X=x,D=j}\rangle_\mathcal{G},\\
		\left[\hat{b}(x)\right]_i&=\langle\hat{\mu}_{Y(0)|X=x,D=i},\hat{\mu}_{Y(0)|X=x,D=0}\rangle_\mathcal{G}\overset{p}{\to}\left[b(x)\right]_i=\langle\mu_{Y(0)|X=x,D=i},\mu_{Y(0)|X=x,D=0}\rangle_\mathcal{G}.
	\end{align*}
	Consequently, Assumption \ref{assumption:B.2} (iv) holds such that \begin{align*}
		\sup _{h \in \mathcal{N}_{n_T}, v \in \mathcal{U}_{n_T}} \frac{1}{n_T} \sum_{i=1}^{n_T}\left|\widehat{\Delta}\left(Z_i, h\right)[v]-\Delta\left(Z_i, h\right)[v]\right|^2=o_p\left(1\right)	    
	\end{align*}
	
	According to Theorem 3.1 of \citet{chen2015sieve}, we obtain that
	\begin{align*}
		\hat{V}_{\theta,P}\overset{p}{\to}V_{\theta}.
	\end{align*}
	
	Based on Theorem \ref{theorem: asymptotic normality} and the Slutsky theorem, we have that
	\begin{align*}
		\sqrt{n_T}\hat{V}_{\theta}^{-\frac{1}{2}}(\hat{\theta}-\theta_0)\overset{d}{\to} N(0,1).
	\end{align*}
\end{proof}

\section{Additional Experimental Results}\label{subsection: simulated experiments}

As mentioned in the main text, we consider three source populations such that $N=3$. We first generate the covariate data. For each source population, we generate the covariate variable $X\overset{\mathrm{iid}}{\sim}\mbox{Uniform}(-1,3)$. For simplicity, we let the treatment group and the control group share the same covariate distribution. Because the target population has some different properties from the source populations, we generate the covariate variable of the target population $X\overset{\mathrm{iid}}{\sim}\mbox{Truncnorm}(mean=0,sd=1,a=-1,b=3)$, where $\mbox{Truncnorm}(mean=0,sd=1,a=-1,b=3)$ denotes the truncated normal distribution from the standard normal distribution with minimum value -1 and maximum value 3.


To evaluate the performance under a wide range of experimental conditions, we generate the parameters of the data generating process from the normal distribution for each iteration. For simplicity, we set the sample size of the treatment group of each source population to be 4000. The sample size of the control group of each source population and the target population varies from 2000 to 4000. For each sample size specification, we repeat the experiments 100 times. In other words, we set the total number of iterations $M=100$. Each time we generate the parameters of the data generating process from the normal distribution. 

After we generate the covariate data, we generate the control group data based on them. For the mixture coefficients $\pi_{i,j}(x)=\frac{exp(e_{i,j}x+f_{i,j})}{\sum_{k=1}^{m}exp(e_{i,k}x+f_{i,k})}$, we set $e_{i,j}=0$ and $f_{i,j}=\ln{(0.8)}\times 1_{i= j}+\ln{(0.1)}\times 1_{i\neq j}$. In other words, if $i=j$ then $\pi_{i,j}(x)=0.8$; if $i\neq j$ then $\pi_{i,j}(x)=0.1$. Consequently, the control group distributions of the source populations have a common set of mixture components, but they have different mixture weights such that
\begin{align*}
	Y(0)|X=x,D=1\sim 0.8N(a_1x+b_1,1)+0.1N(a_2x+b_2,1)+0.1N(a_3x+b_3,1),\\
	Y(0)|X=x,D=2\sim 0.1N(a_1x+b_1,1)+0.8N(a_2x+b_2,1)+0.1N(a_3x+b_3,1),\\
	Y(0)|X=x,D=1\sim 0.8N(a_1x+b_1,1)+0.1N(a_2x+b_2,1)+0.1N(a_3x+b_3,1).\\
\end{align*}
For the component distribution $N(a_jx+b_j,1)$, we let $a_j, b_j\overset{\mathrm{iid}}{\sim}N(0,15^2)$ for each round of experiment, where $j=1, \ldots, m$.

We follow the steps mentioned in the main text to generate the treatment group data through the counterfactual expectation transition mechanism. For the counterfactual expectation transition mechanism 
\begin{align*} 
	Y(1)=\mathbb{E}[Y(1)|Y(0)=y,X=x,D=0]+\epsilon=g_1y+g_2x+g_3xy+\epsilon,		
\end{align*}
we let $g_1, g_2, g_3\overset{\mathrm{iid}}{\sim}N(0,10^2)$ and $\epsilon\overset{\mathrm{iid}}{\sim}N(0,1)$ for each round of experiment.

As for the target population, recall that the control potential outcome is generated by the mixture normal distribution
\begin{equation*}
Y(0)|X=x,D=0\sim\sum_{j=1}^{m}(\sum_{i=1}^{N}w_i(x)\pi_{i,j}(x))N(a_jx+b_j,1).
\end{equation*}
For the weight $w_i(x)=\frac{exp(c_ix+d_i)}{\sum_{j=1}^{N}exp(c_jx+d_j)}$, we set $c_i \overset{\mathrm{iid}}{\sim}N(0,1)$ and $d_i \overset{\mathrm{iid}}{\sim}N(0,1.5^2)$ for each round of experiment, where $i=1, \ldots, N$.

We use the Gaussian kernel to estimate the conditional mean embedding and set the hyperparameter $\lambda$ in Eqn.~\ref{equation: notations} to 0.01. We employ the R package: kernlab \citep{Karatzoglou2004kernlabA}. For weight estimation, we adopt the B-spline basis functions with order 3. For outcome regression estimation, we use the B-spline basis functions with order 3 and 2 knots. When we apply the \textbf{Point} method, we replace the sieve extremum estimator $\hat{w}_i(x)$ in Eqn.~\ref{equation: treatment group expectation estimator} with the pointwise weight estimator in Theorem \ref{proposition: estimation}. 


Due to the space constraint, in the main text, we only present the weight estimation performance of the first population in Figure \ref{fig:weight estimation}. In Figures \ref{fig:the second source population} and \ref{fig:the third source population}, we present the weight estimation results of the second and the third population, where the X-axis depicts the Covariate variable $X$ and the Y-axis depicts the weight $w_i(x)$. Recall that the label "source" represents the true value of the weight function. For a better comparison, we combine the weight estimation results of Figures \ref{fig:weight estimation}, \ref{fig:the second source population} and \ref{fig:the third source population} into Figure \ref{fig:Weight estimators for all source populations}. In Figure \ref{fig:Weight estimators for all source populations}, the label "source\_1" refers to the true value of the weight function of the first source population and the label "sieve\_1" refers to the weight estimator based on the sieve extremum estimator of the first source population. 

\begin{figure}
   \begin{subfigure}{0.42\textwidth}
   \includegraphics[width=\linewidth]{"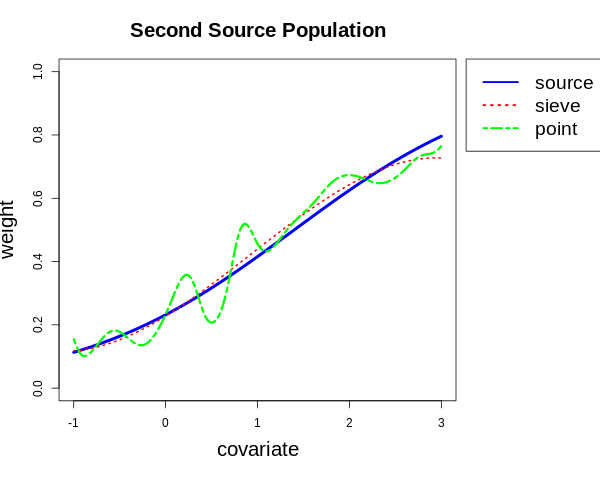"}
    \caption{The second source population.}
    \label{fig:the second source population}
   \end{subfigure}
   \hfill
   \begin{subfigure}{0.42\textwidth}
      \includegraphics[width=\linewidth]{"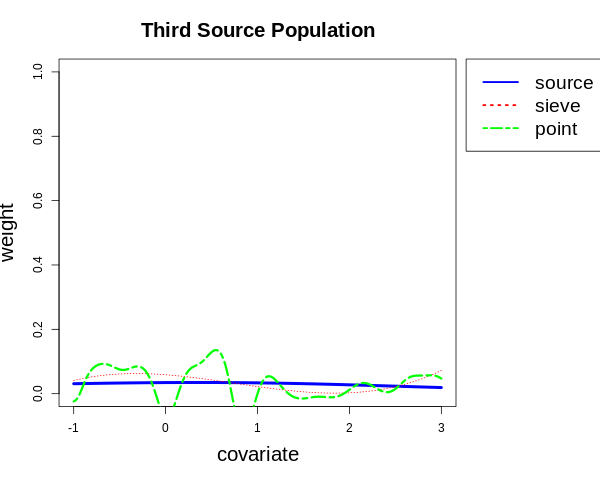"}
      \caption{The third source population.}
      \label{fig:the third source population}
   \end{subfigure}
   \hfill
   \begin{subfigure}{0.42\textwidth}
      \includegraphics[width=\linewidth]{"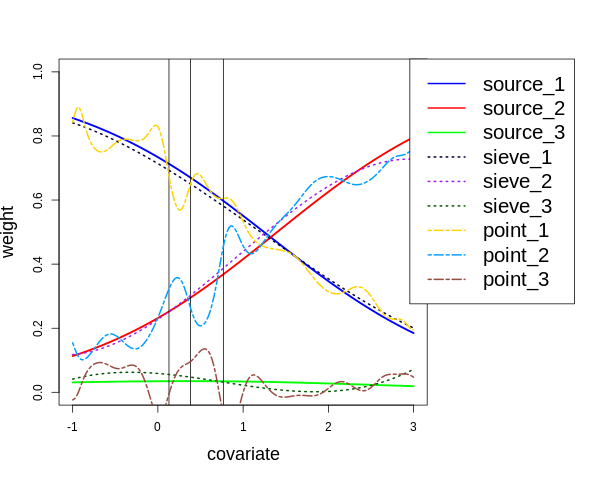"}
      \caption{Weight estimators for all source populations.}
      \label{fig:Weight estimators for all source populations}
   \end{subfigure}
   
   \caption{Weight estimators for the source populations. Figure \ref{fig:weight estimation} in the main text only depicts the weight estimators for the first population. In this figure we show the weight estimators for the other two source populations and combine all the results into Figure \ref{fig:Weight estimators for all source populations}.}
\end{figure}

\begin{figure}
	\centering
	\includegraphics[width=0.45\textwidth]{"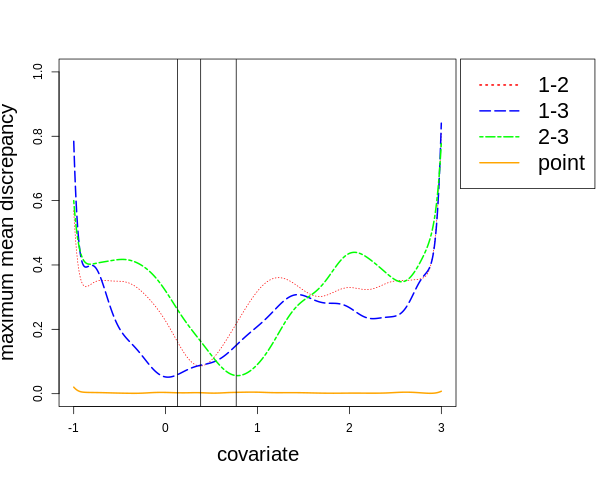"}
	\caption{Estimated CMMD between source populations. The label "point" refers to the CMMD between the weighted source populations and the target population. The label "1\_2" refers to the CMMD between the first population and the second population.}\label{fig:Estimated CMMD between source populations.}
\end{figure}

We compare the performance of the pointwise weight estimation based on the constrained weights (with the constraint that the weights are nonnegative and adds up to one) and unconstrained weights in Figure \ref{fig:comparison between constrained weights and unconstrained weights}, where the label "source" refers to the true value of the weights. As illustrated in Figure \ref{fig:comparison between constrained weights and unconstrained weights}, the pointwise weight estimation method based on constrained weights (the dashed lines) and unconstrained weights (the dotted lines) give similar results in most situations. Consequently, the synthetic treatment group estimators based on them have similar performances. 

Furthermore, we depict the estimated CMMD between different source populations in Figure \ref{fig:Estimated CMMD between source populations.}, which is calculated by Eqn.~\ref{equation: CMMD calculation}. In Figure \ref{fig:Estimated CMMD between source populations.}, the label "1-2" represents the CMMD between the first source population and the second source population.

As mentioned above, for each round of the experiments we generate the parameters of the data generating process from the normal distribution in order to evaluate the performance under a wide range of experimental conditions. Due to the space constraint, in the main text, we only present the weight estimation results under a particular set of parameters of the data generating process. In Figures \ref{fig:Weight estimators for the first source population}, \ref{fig:Weight estimators for the second source population} and \ref{fig:Weight estimators for the third source population}, we present additional experimental results to demonstrate the weight estimation performance under a different parameter specification.

\begin{figure}
	\centering
\includegraphics[width=0.55\textwidth]{"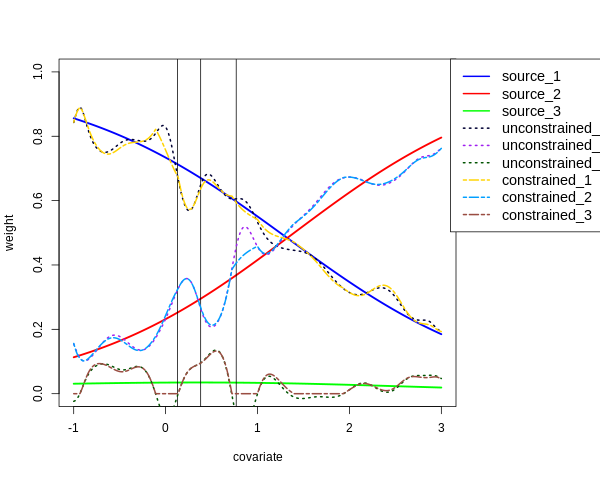"}
	\caption{Comparison between pointwise weight estimators based on constrained weights and unconstrained weights. The label "contrained\_1" refers to the constrained weight estimator for the first population. The label "source\_1" represents the true value of the weight for the first source population.}\label{fig:comparison between constrained weights and unconstrained weights}
\end{figure}

\begin{figure}[h]
   \begin{subfigure}{0.45\textwidth}
      \includegraphics[width=\linewidth]{"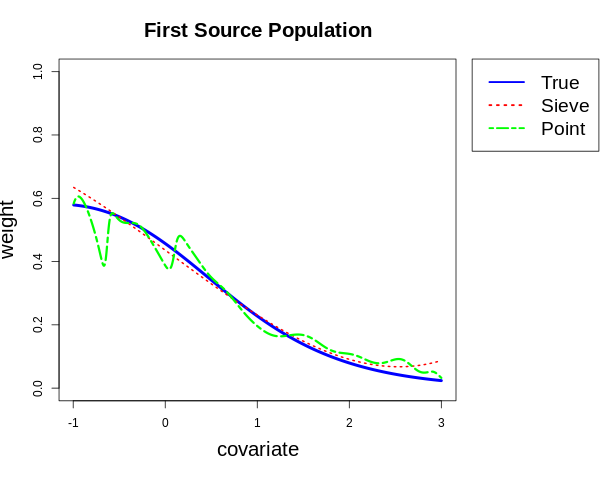"}
      \caption{The first source population.}
      \label{fig:Weight estimators for the first source population}
   \end{subfigure}
   \hfill
   \begin{subfigure}{0.45\textwidth}
      \includegraphics[width=\linewidth]{"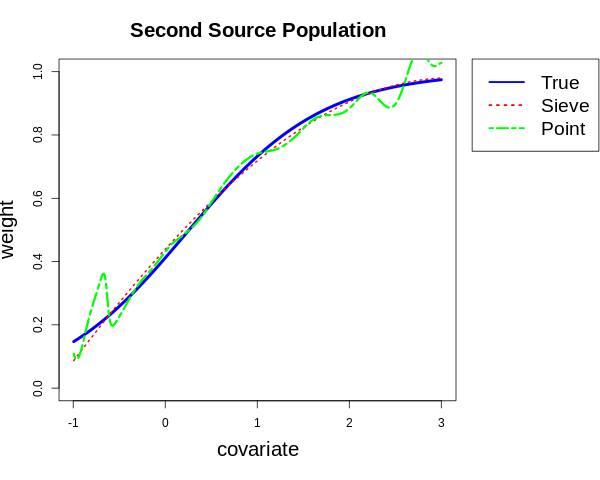"}
      \caption{The second source population.}
      \label{fig:Weight estimators for the second source population}
   \end{subfigure}
   \hfill
   \begin{subfigure}{0.45\textwidth}
       \includegraphics[width=\linewidth]{"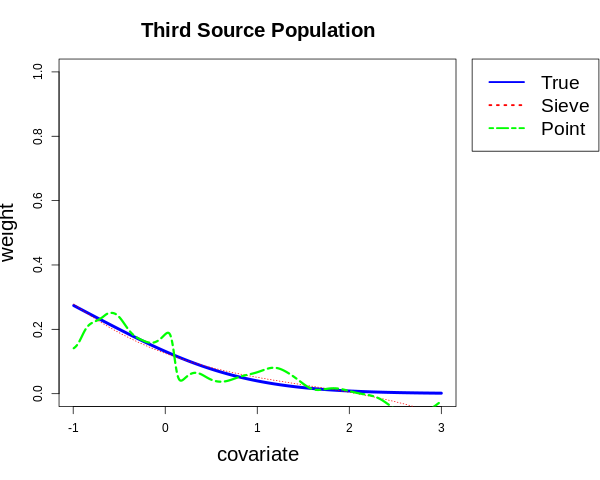"}
       \caption{The third source population.}
       \label{fig:Weight estimators for the third source population}
   \end{subfigure}
   \hfill
   \begin{subfigure}{0.45\textwidth}
       \includegraphics[width=\linewidth]{"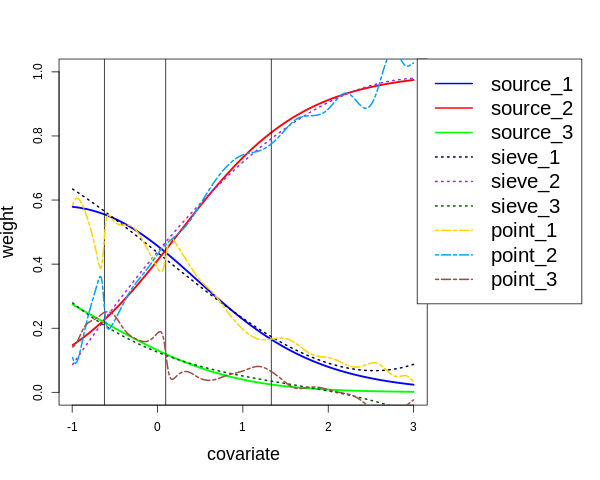"}
       \caption{Combining the results for all source population.}
       \label{fig:Weight estimators for all source population.}
   \end{subfigure}
   \caption{Weight estimators under a different parameter specification.}
\end{figure}

For a better comparison, we combine the weight estimation results under a different parameter specification from Figures \ref{fig:Weight estimators for the first source population}, \ref{fig:Weight estimators for the second source population} and \ref{fig:Weight estimators for the third source population} into Figure \ref{fig:Weight estimators for all source population.}. As shown in Figure \ref{fig:Weight estimators for all source population.}, the weight estimators based on the sieve extremum estimation method lead to smooth results (see the dotted lines) while the pointwise weight estimators result in wiggly results (see the dashed line). It confirms with the finding that the sieve extremum estimation method can alleviate the problem of overfitting caused by the pointwise weight estimation.

The CMMD between two source populations can reflect the difficulty of the weight estimation. Figure \ref{fig:Estimated CMMD between the control group distributions of different source populations.} depicts the estimated CMMD between different source populations, where the label "1-2" refers to the estimated CMMD between the first source population and the second source population. As shown by the red dotted curve in Figure \ref{fig:Estimated CMMD between the control group distributions of different source populations.}, the CMMD between the first and the second source population approaches to zero near the left vertical line. It indicates that they have similar control group distributions near the left vertical line where it is difficult to distinguish them. Correspondingly, in Figure \ref{fig:Weight estimators for all source population.}, the pointwise weight estimators of the first population (yellow dashed line) and the second source population (blue dashed line) oscillate drastically near the left vertical line. Similarly, in Figure \ref{fig:Estimated CMMD between the control group distributions of different source populations.}, the CMMD between the first source population and the third source population approaches to zero near the middle vertical line. It corresponds to the fact that in Figure \ref{fig:Weight estimators for all source population.} the pointwise weight estimators of the first source population (yellow dashed line) and the third source population (brown dashed line) become increasingly wiggly near the middle vertical line. In contrast, in Figure \ref{fig:Weight estimators for all source population.}, the weight estimator based on the sieve extremum estimation (dotted curves) remains smooth near the left and the middle vertical lines, which demonstrates that the sieve extremum estimation method can alleviate the problem of overfitting.

\begin{figure}[H]
	\centering
	\includegraphics[width=0.5\textwidth]{"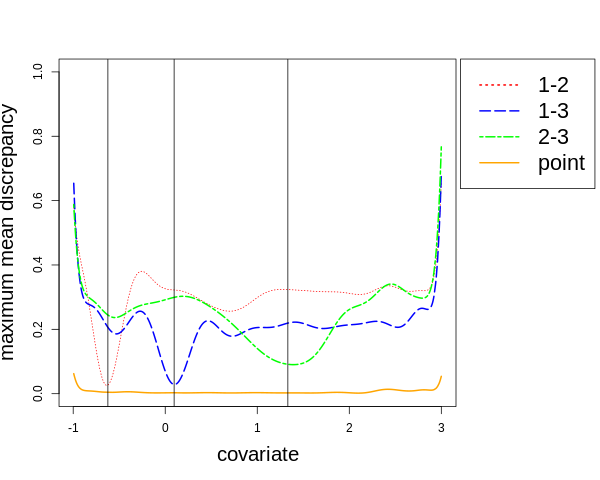"}
	\caption{Estimated CMMD between the control group distributions of different source populations under a different parameter specification.}\label{fig:Estimated CMMD between the control group distributions of different source populations.}
\end{figure}

Since the ultimate goal is to estimate the treatment group expectation $\theta_0$, the weight estimation is an intermediate step. The minor deviation of the weight estimator from the true value in Figure \ref{fig:Weight estimators for all source population.} doesn't necessarily mean a bad thing. If two control group distributions are similar to each other, then it is difficult to estimate the corresponding weights. But as long as the sum of these two weights approximates the true value, the resulting estimator still has a reasonably good performance.

\end{document}